\documentclass[sn-mathphys-ay]{sn-jnl}
\usepackage{amssymb}
\usepackage{amsmath}
\usepackage{amsthm}
\usepackage{tikz-cd}
\usepackage{graphicx}
\graphicspath{ {./graphics_for_paper/} }
\usepackage{wrapfig}
\usepackage{caption}
\usepackage{subcaption}
\usepackage{enumitem}

\usepackage{algorithm}
\usepackage{algpseudocode}
\usepackage{mathrsfs}

\newtheorem{definition}{Definition}
\newtheorem{theorem}{Theorem}[section]
\newtheorem{proposition}[theorem]{Proposition}
\newtheorem{lemma}[theorem]{Lemma}

\theoremstyle{definition}
\newtheorem{example}{Example}[section]

\newcommand{\R}{\mathbb R}

\newcommand{\C}{\mathbb C}
\newcommand{\Img}{\text{Im } }
\newcommand{\Ker}{\text{Ker } }

\newcommand{\spn}{\mathrm{span}}

\newcommand{\changedSB}[1] {#1} %{{\color{blue}{#1}}} 
 
\newcommand{\changedSY}[1] {#1} %{{\color{teal}{#1}}}

\newcommand{\changedSBa}[1] {#1} %{{\color{blue}{#1}}} 

\title{\changedSB{Weighted Combinatorial Laplacian and its Application to Coverage Repair in Sensor Networks}}
\author{Shunsaku Yadokoro \quad and \quad Subhrajit Bhattacharya 
\footnote{Department of Mechanical Engineering and Mechanics, Lehigh University, Bethlehem, PA, U.S.A. email: \texttt{[shy523,sub216]@lehigh.edu}.
	
We gratefully acknowledge the support of Air Force Office of Scientific Research (AFOSR) award number FA9550-23-1-0046.}
}
\date{}

\begin{document}

\maketitle

\begin{abstract}
    We define the weighted combinatorial Laplacian operators on a simplicial complex and investigate their spectral properties. %Extremal eigenvalues 
    \changedSB{Eigenvalues close to zero}
    and the corresponding eigenvectors of them are especially of our interest, and we show that they can detect almost $n$-dimensional holes in the given complex.
    \changedSB{Real-valued weights on simplices allow gradient descent based optimization, which in turn gives}
    % We then apply these results to give 
    an efficient dynamic coverage repair algorithm for the sensor network of \changedSB{a} mobile robot team.
    \changedSBa{Using the theory of relative homology, we also extend the problem of dynamic coverage repair to environments with obstacles.}
\end{abstract}

\vspace{4em}
% \newpage 

\section{Introduction \changedSB{and Related Work}}

\subsection{\changedSB{Motivation and Contribution}}

\changedSB{Graphs have been used extensively to model networks of robot or sensor networks~\citep{newman2006structure,CVETKOVIC20102257}. One of the fundamental algebraic tools relevant to graphs is the graph Laplacian matrix, the spectrum of which encodes the connectivity of the graph~\citep{godsil2001algebraic}. In \emph{weighted graphs}, one assigns non-negative, real-valued weights or importance to each edge of the graph, with a zero weight on an edge being equivalent to the edge being non-existent, thus allowing a continuum between different graph topologies. Furthermore, for robot or mobile sensor networks, the real-valued weights naturally correspond to the separation or distance between pairs of agents (with the weights being inversely related to the distances so that agents that are closer to each other are strongly connected, while agents that are farther from each other are weakly connected). %This continuum has far-reaching consequences in terms of designing gradient descent type optimization algorithms for changing the graph topology in a desired way.
A \emph{weighted graph Laplacian} can be constructed accordingly, and real-valued optimization objectives can be formulated to control the connectivity of a network~\citep{zavlanos2008distributed,leiming-subhrajit}.

On the other hand, simplicial complexes, which are a natural higher-dimensional extension of graphs, have also been used to model networks of robots and mobile sensors~\citep{derenick}. A simplicial complex is accompanied by a corresponding algebraic construction called \emph{combinatorial Laplacian} matrix, the spectrum of which is related to the \emph{holes} in the network. However, the definition of combinatorial Laplacian cannot naturally incorporate weights on the simplices and hence there does not exist any natural continuum between different simplicial complex topologies. While the eigenvectors of the combinatorial Laplacian have been used to identify holes in coverage of mobile sensor networks~\citep{derenick}, the control of the mobile sensors for mending such holes involved finding the \emph{tightest cycle} around the holes and then navigating the mobile sensors towards the holes.

The main technical contribution of this paper is a new theory of a \emph{weighted combinatorial Laplacian} of a simplicial complex that can encode real-valued weights on the simplices, providing a continuum between different simplicial complex topologies. This allows us to design gradient descent algorithms for directly controlling the mobile sensors to optimize objective functions based on the spectrum of the weigted combinatorial Laplacian for mending holes in the complex.
}

\subsection{Overview}

\begin{wrapfigure}{r}{0.25\textwidth} \vspace{-5em}
    \centering
    \includegraphics[width=0.25\textwidth]{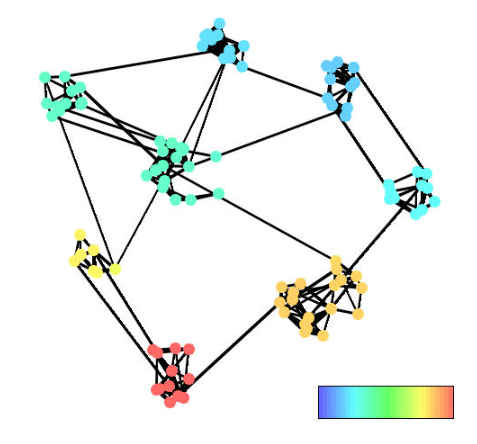}
    \caption{The Fiedler vector of a graph} \label{fig:fielder}
    \vspace{-2em}
\end{wrapfigure}
\emph{\textbf{\changedSBa{Weighted Graph Laplacian:}}}
The important special case for the theory we provide in this paper, known as the \textit{weighted graph Laplacian operator}, has been well-studied and has found wide range of applications in computational geometry and network theory \citep{leiming-subhrajit}. The graph Laplacian operator can be constructed as a linear operator between the vector spaces over a field (typically $\R$ or $\C$) with basis the vertices of the graph. It was Fiedler~\citep{fiedler} who first noticed that the smallest non-zero eigenvalue of the graph Laplacian has some implication about the connectivity of the graph. 
%\noindent
Figure~\ref{fig:fielder} shows the coloring of the vertices of a graph based on the eigenvector corresponding to the smallest nonzero eigenvalue of the weighted graph Laplacian operator, often called the \textit{Fiedler vector}.
Edges of the graph are weighted, that is, each edge is assigned a real value, and in this particular example, weights are set to be the inverse of the length of the edge, giving the notion of ``the shorter the edge, the stronger the connection.'' 
Then it turns out that the resulting weighted Laplacian operator, and in particular the Fiedler vector, show how much the graph is ``almost disconnected,'' in the sense that the difference of colors between a pair of vertices measures the connection between the two vertices through the graph. Hence, by utilizing this ``almost disconnectedness,'' we can capture the clusters in the graph, the partition of vertices with similar colors. 

 \vspace{0.5em}
\noindent\emph{\textbf{\changedSBa{Weighted Laplacians and ``Almost Holes'':}}}
On the other hand, in a branch of mathematics called algebraic topology, we can apply to a graph (or more generally a simplicial complex) what is called homology theory, and for each non-negative integer $n$, we are able to measure the $n$-dimensional holes in the graphs (or the simplicial complexes) by computing their $n$th homology. Figure \ref{fig:illustration of holes} informally illustrates the idea of the $n$-dimensional holes in a simplicial complex for small $n$.
The important thing to

\begin{wrapfigure}{l}{0.35\textwidth} \vspace{-3em}
	\centering
	\includegraphics[scale=0.25]{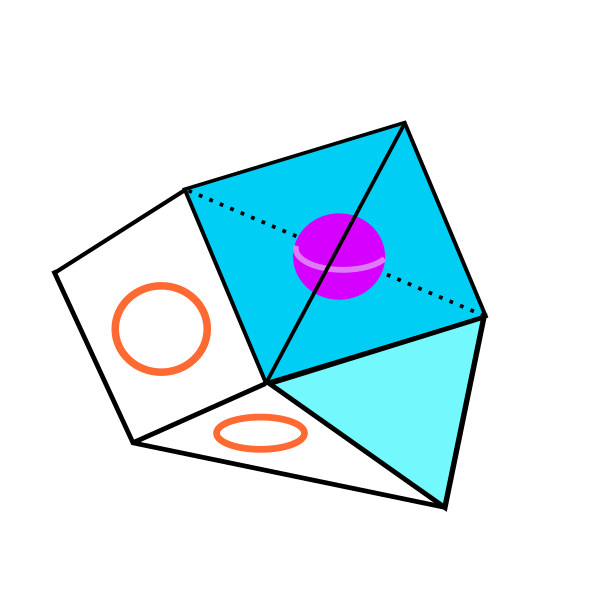}
	\caption{Illustration of holes in a simplicial complex. There are two 1-dimensional holes (shown as orange circles) and one 2-dimensional hole (shown as a purple sphere)}
	\label{fig:illustration of holes} \vspace{-2em}
\end{wrapfigure}
\noindent
note here is that the disconnectedness of a graph (or a simplicial complex) is regarded as a 0-dimensional hole in the homology theory, as the disconnectedness is measured completely by the 0th homology. Hence, weighted graph Laplacian operator discussed above is essentially measuring the ``almost 0-dimensional holes'' in the graph. Our purpose of this paper is to developing an analogue of weighted graph Laplacian for each dimension $n$ so that we can capture the ``almost $n$-dimensional holes'' in a simplicial complex. In this sense, the weighted graph Laplacian can be regarded as the 0th operator of our weighted combinatorial Laplacian operators. 

\vspace{0.5em}
\noindent\emph{\textbf{\changedSBa{Illustration:}}}
To illustrate what our weighted combinatorial Laplacian can measure, let us take a look at the case $n=1$ and explain how our weighted Laplacian can capture ``almost 1-dimensional holes''. The Figure \ref{fig:complex} is a simplicial complex (a \changedSBa{higher-dimensional extension of graphs} with \emph{\changedSBa{filled-in} triangle \changedSBa{elements}}). Since, \changedSBa{in this example}, each triangle is filled, combinatorially there is no 1-dimensional holes present in this complex, and in fact we can verify this fact by computing its 1st homology. However, one can tell that there are some parts of the complex where the vertices and the edges are scattered, and this subtlety can be captured well by our weighted Laplacian operators: The histogram shown in Figure \ref{fig:histrogram} shows the comparison of the smallest eigenvalues of our 1st weighted Laplacian operator. From this, one may notice that there is a \changedSBa{relatively large} gap between the third and fourth eigenvalue, and this is how our weighted Laplacian tells us that there are three ``almost 1-dimensional holes'' in the complex. Furthermore, the corresponding eigenvectors of the three eigenvalues \changedSBa{indicate} where these three holes are placed in the simplicial complex: In Figure \ref{fig:three eigenvecs}, the three eigenvectors are represented by colors of the edges, and it is clear which edges are forming the ``almost 1-dimensional holes''. 

\begin{figure}[H]
\centering
\begin{subfigure}[b]{0.5\textwidth}
\centering
\includegraphics[width=0.75\textwidth]{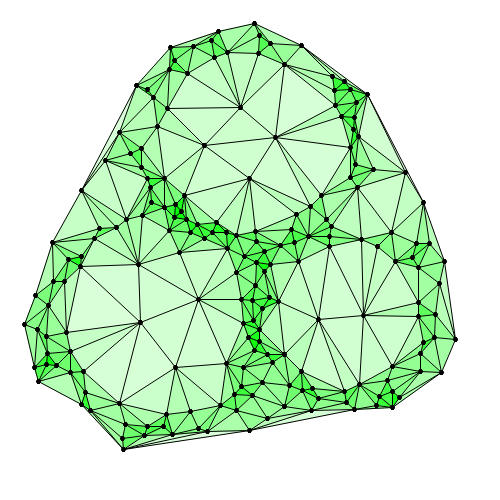} 
\caption{Simplicial complex \changedSB{constructed from Delaunay triangulation of a set of sensors on the Euclidean plane. Weight on a edge is inversely related to the distance between the vertices that make up the edge, and weight on a triangle is related to the product of \changedSBa{the weights on} its boundary edges. \changedSBa{Colors on the triangles indicate the weights on the $2$-simplices -- lighter color indicates lower weight, while darker color imply higher weights.}}}
\label{fig:complex}
\end{subfigure}
\hspace{1em}
\begin{subfigure}[b]{0.45\textwidth}
\centering
\includegraphics[width=0.8\textwidth]{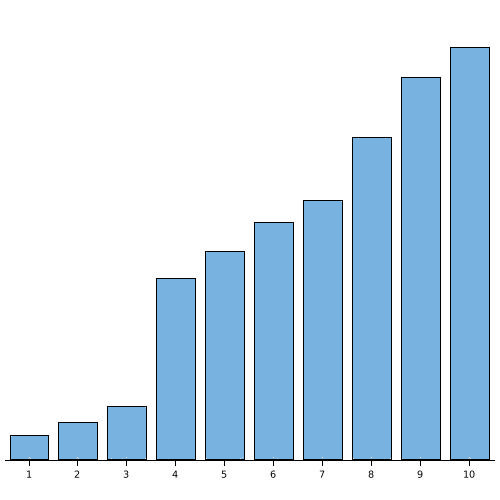}
\caption{The smallest eigenvalues of our \changedSB{proposed} 1st \changedSB{weighted} Laplacian operator, \changedSB{$\tilde{\mathcal{L}}_1$}. Note how the first three eigenvalues are significantly smaller compared to the others\changedSY{, and this corresponds to the fact that there are three ``almost holes'' in the complex.}}
\label{fig:histrogram}
\end{subfigure}
\hfill
\begin{subfigure}[b]{0.9\textwidth}
\includegraphics[width=\textwidth]{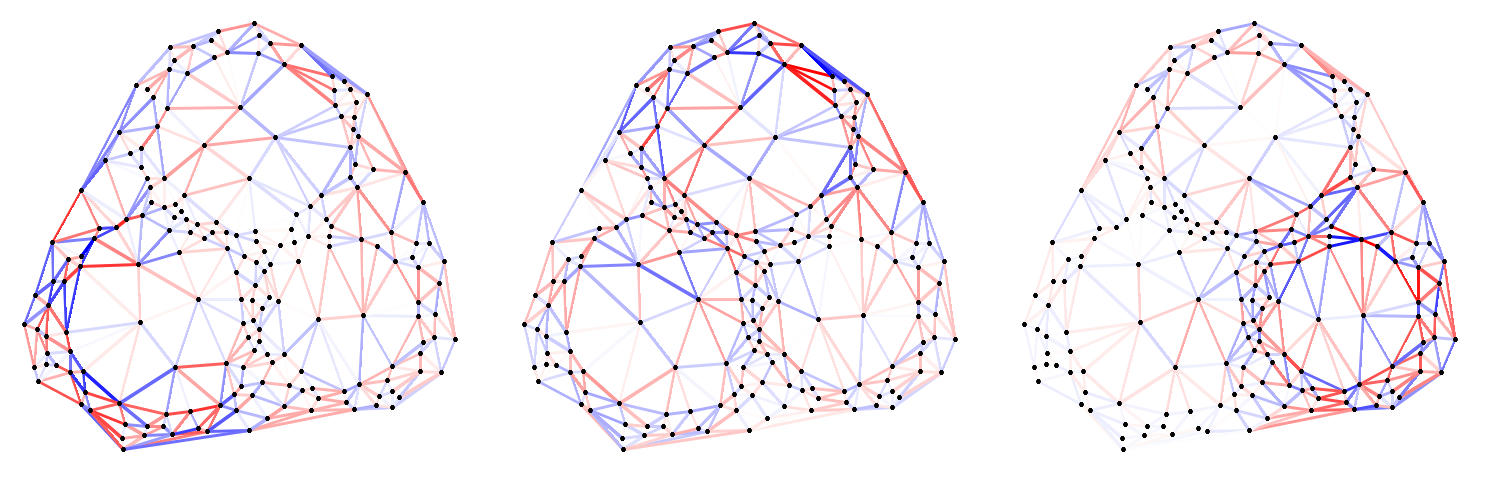}
\caption{Eigenvectors of graphs corresponding to the smallest three eigenvalues.}
\label{fig:three eigenvecs}
\end{subfigure}
\caption{\changedSB{\emph{\changedSBa{Preview/Overview of the main technical contribution (actual result using our proposed Weighted Combinatorial Laplacian)}:} A demonstration} of \changedSB{the proposed} 1st weighted Laplacian operator, \changedSB{$\tilde{\mathcal{L}}_1$. Although the simplicial complex does not literally have a hole in it (and hence the null-space of $\tilde{\mathcal{L}}_1$ is empty), the fact that the sensors\changedSBa{/vertices} are placed in the shape of three touching circular rings, creates low-weight simplices in the interior of the rings. This manifests as the three small eigenvalues of $\tilde{\mathcal{L}}_1$ shown in (b).}}
\label{fig:example of l1}
\end{figure} 

Another advantage of our weighted Laplacian operators is the fact that the eigenvalues are the real-valued piece-wise smooth functions of weights, whose gradient computation is not too expensive even for a large simplicial complex, as we shall see. Hence, given that we have some control over the dynamics of the simplicial complex, say, positions of the vertices, our Laplacian operators will give us a good control over the holes of the simplicial complex. In this way, we will apply our weighted Laplacian operators to develop the efficient coverage repair algorithm for sensor networks.

\vspace{0.5em}
\noindent\emph{\textbf{\changedSBa{Organization:}}}
This paper is structured as follows: In Section~\ref{sec:background}, we provide some background information regarding simplicial complex, homology theory and the combinatorial Laplacian. \changedSBa{Section~\ref{sec:weighted-laplacian} contains the main technical contribution of the paper, where} we define our weighted combiantorial Laplacian and investigate its properties. In Section~\ref{sec:application}, we switch our gear to applications and develop the algorithm for dynamic coverage repair of sensor network.

\section{Background}\label{sec:background}
In this section, \changedSBa{after introducing come notations and conventions}, we provide a brief overview of \textit{homology theory} and \textit{Combinatorial Laplacian}, which are the basic machinery of our method. We first introduce \textit{Simplicial Complex}, the higher-dimensional analogue of graphs. We then start discussing how the homology vector spaces can be defined on a simplicial complex and how it can measure the holes in the simplicial. Only after that we introduce the combinatorial Laplacian and list its fundamental properties, which we will make a heavy use of in the following section.
\changedSBa{For more details on simplicial complexes, homology theory and combinatorial Laplacian operators, the reader can refer to a standard text on Algebraic Topology such as~\citep{hatcher,ghrist2014elementary}.}

\subsection{Notations and Conventions}
The norm of a vector and a linear transformation (or a matrix) will be denoted by $\| \cdot \|$. Unless otherwise noted, the norm is the two norm. The inner product of the two vectors will be denoted by $\langle \cdot, \cdot \rangle$. 

\changedSB{Given a linear operator, $\mathcal{O}: \mathscr{A} \rightarrow \mathscr{B}$, its adjoint with respect to an inner product is denoted by $\mathcal{O}^\ast$.
For a given basis for the vector spaces, if $O$ represents the matrix representation of the linear operator, then with the standard Euclidean metric with respect to the basis, the matrix representation of the adjoint of the operator is the transpose, $O^T$.}

In the literature of the graph Laplacian operators and the combinatorial Laplacian operators, it is more common to define them over cochain complexes than chain complexes. However, we will present our combinatorial Laplacian in terms of chain complexes with real coefficients, so that we explain the properties of Laplacian in terms of the holes of the simplicial complex. Note that, because we will keep our simplicial complex to be finite, homology with real coefficients and cohomology with real coefficients are essentially the same, as there will always be an isomorphism between the two when defined on the same simplicial complex.

\subsection{\changedSB{Weighted Graph, Incidence Matrix and Weighted Graph Laplacian}}

%\changedSB{
\vspace{0.5em}\noindent
\emph{\textbf{Graph and Boundary Map:}} An (undirected) graph consists of a finite set, $V$, the elements of which are called \emph{vertices}, and a collection of $2$-element subsets of $V$ called the \emph{edge set}, $E$.
Upon assigning an arbitrary orientation to each edge, one can define a linear map between an
$|E|$-dimensional vector space (which can be interpreted as a vector space spanned by the edges) and
a $|V|$-dimensional vector space (which can be interpreted as a vector space spanned by the vertices\footnote{For technical reasons, this vector space is constructed as a span of singletons ($1$-element subsets) of $V$ rather than the elements of $V$ themselves. However, for simplicity, for now, we denote both the vertex set and the set containing all the singletons of the vertex set by $V$.})
called the \emph{incidence map} or the \emph{boundary map}, 
${\mathcal B}: \spn(E) \rightarrow \spn(V)$,
which has the following property:
If $\{x,y\}\in E$ is an edge emanating from $x\in V$ to $y\in V$ (according to the arbitrarily-chosen orientation), then
\[ {\mathcal B}(\{x,y\}) = \{y\} - \{x\} \]
Using $E$ and $V$ as the basis for the domain and co-domain of ${\mathcal B}$, one can write a matrix representation of ${\mathcal B}$, which is called the \emph{incidence matrix} or the \emph{boundary matrix}, $B\in\mathbb{R}^{|V|\times|E|}$, the elements of which can be described explicitly as
\[ B_{ik} = \left\{ \begin{array}{l} 
1, \quad\text{if the $k$-th edge is incident into the $i$-th vertex}, \\
-1, \quad\text{if the $k$-th edge is emanating from the $i$-th vertex}, \\
0, \quad\text{otherwise}
\end{array} \right.\]
An explicit example of a boundary matrix, $B_1$, is shown in Figure~\ref{fig:boundary-matrix-illustration}.

\vspace{0.5em}\noindent
\emph{\textbf{Weighted Graph and Weighted Boundary Map:}}
One can assign weights or importance to the edges (with a zero weight being equivalent to a non-existent edge). We thus define the weight function, $w:E \rightarrow \mathbb{R}_{\geq 0}$, and a corresponding weighted insidence/boundary map $\tilde{{\mathcal B}}: \spn(E) \rightarrow \spn(V)$ such that $\tilde{{\mathcal B}}(\{x,y\}) = w(\{x,y\}) \, \big(\{y\} - \{x\}\big)$.
% More compactly, defining the weight
The matrix representation of this map is thus given by the $|V|\times|E|$ matrix, $\tilde{B}$, whose elements are
\[ \tilde{B}_{ik} = \left\{ \begin{array}{l} 
w(e_k), \quad\text{if the $k$-th edge is incident into the $i$-th vertex}, \\
-w(e_k), \quad\text{if the $k$-th edge is emanating from the $i$-th vertex}, \\
0, \quad\text{otherwise}
\end{array} \right.\]
where $e_k$ is the $k$-th edge in $E$.

More compactly, defining the weight operator, $\mathcal{W}: \spn(E) \rightarrow \spn(E)$, such that $\mathcal{W}(\{x,y\}) = w(\{x,y\}) \, \{x,y\}$, we have $\tilde{{\mathcal B}} = {\mathcal B}\, \mathcal{W}$. If $W$ is the corresponding matrix representation of the weight operator (which is a diagonal matrix in the basis of the simplices), we have $\tilde{B} = B W$.

% , and edges, $E\subseteq V\times_{\text{symm}} V$. A weight on the edges, $w:E\rightarrow\mathbb{R}_{+}$, assigns weights or importance to the edges (with a zero weight being equivalent to a non-existent edge). For a graph with $n$ vertices and $m$ edges we can define a weighted incidence matrix, $B \in \mathbb{R}^{n\times m}$, as follows:
% Start by assigning an arbitrary orientation/direction to each edge. 
% % Then, suppose the $k$-th edge points from the $i$-th vertex to the $j$-th vertex, then the $k$-th column of $B$ contains a $1$ at th $j$-th place and a $-1$ at the $i$-th place, and zero everwhere else. 
% Then,
% \[ B_{ik} = \left\{ \begin{array}{l} 
% 1, \quad\text{if the $k$-th edge is incident into the $i$-th vertex}, \\
% -1, \quad\text{if the $k$-th edge is emanating from the $i$-th vertex}, \\
% 0, \quad\text{otherwise}
% \end{array} \right.\]
% $B$ can be interpreted as a linear map from a $m$-simensional vector space spanned by the edges, $\bar{E} = \text{span}(\mathbf{e}_1,\mathbf{e}_2,\cdots,\mathbf{e}_m)$, to a $n$-dimensional vector space spanned by the vertices, $\bar{V} = \text{span}(\mathbf{v}_1,\mathbf{v}_2,\cdots,\mathbf{v}_n)$, such that, if the $k$-th edge emanates from the $i$-th vertex and is incident into the $j$-th vertex, we have $B \, \mathbf{e}_k = \mathbf{v}_j - \mathbf{v}_i$.

% For notational convinience, 

\vspace{0.5em}\noindent
\emph{\textbf{Graph Laplacian and Weighted Graph Laplacian:}}
The graph Laplacian is a linear map $\mathcal{L}: V \rightarrow V$, and is defined as $\mathcal{L} = {\mathcal B} \, {\mathcal B}^\ast$.
% \footnote{Where ${\mathcal B}^\ast$ refers to the adjoint of ${\mathcal B}$, and with standard Euclidean norm on the vector spaces, is equivalent to the \emph{transpose} of ${\mathcal B}$}. 
In terms of the matrix representation the Laplacian matrix is $L = B B^T$.
The spectrum of the Laplacian matrix has been studied extensively in algebraic graph theory and network theory~\citep{}.
It is a well-known fact that the Laplacian is a symmetric, positive semi-definite operator that is independent of the choice of orientation on the edges, and that the number of connected components of a graph is equal to the dimension of the null space (the multiplicity of the zero eigenvalue) of the Laplacian.
%All eigenvalues of the Laplacian are non-negative and, 
% For a single connected graph, the second smallest eigenvalue \changed{(often referred to as the \textit{Fiedler \changedB{value}} \citep{fiedler1973algebraic})} indicates how connected the graph is.

\begin{figure} %{wrapfigure}{r}{0.6\columnwidth}
	%\vspace{-2.5em}
	\centering
%	\hfill
	\begin{subfigure}[b]{0.9\textwidth}\centering
		\includegraphics[width=0.85\columnwidth,trim=0 0 0 0,clip=true]{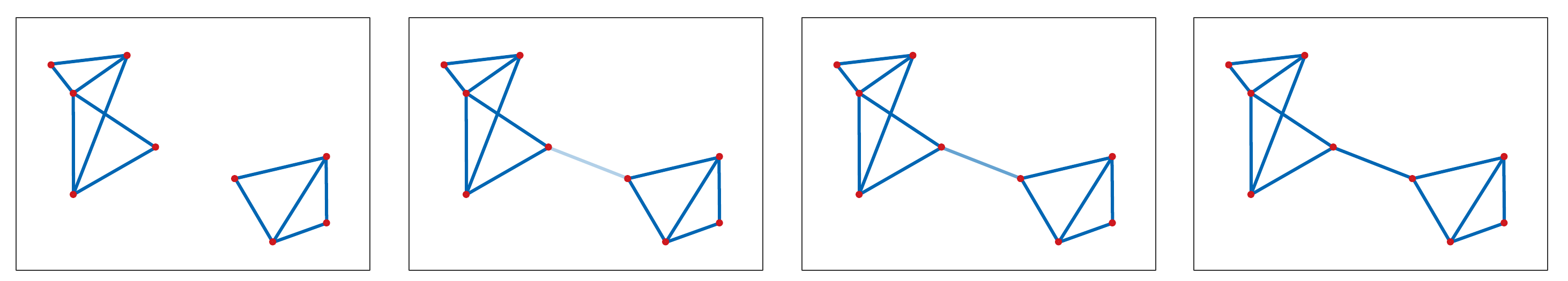}
		\caption{A continuum between fully disjoint graph ($2$-dimensional null-space of weighted graph Laplacian, $\tilde{\mathcal{L}}_0$ due to presence of two components) and a connected graph ($1$-dimensional null-space of weighted graph Laplacian, $\tilde{\mathcal{L}}_0$).}
		\label{fig:weighted-graph}
	\end{subfigure}
	%
%	\subfloat[\changedSB{A continuum between fully disjoint graph ($2$-dimensional null-space of weighted graph Laplacian, $\tilde{\mathcal{L}}_0$ due to presence of two components) and a connected graph ($1$-dimensional null-space of weighted graph Laplacian, $\tilde{\mathcal{L}}_0$).
% % \todo{add spectrum below each figure?}
% }]{\hspace{5em}{\includegraphics[width=0.75\columnwidth,trim=0 0 0 0,clip=true]{figures/weighted_graph.pdf}\hspace{5em}} \label{fig:weighted-graph}}
%	%	
		\hfill
	\begin{subfigure}[b]{0.9\textwidth}\centering
		\includegraphics[width=0.85\columnwidth,trim=0 0 0 0,clip=true]{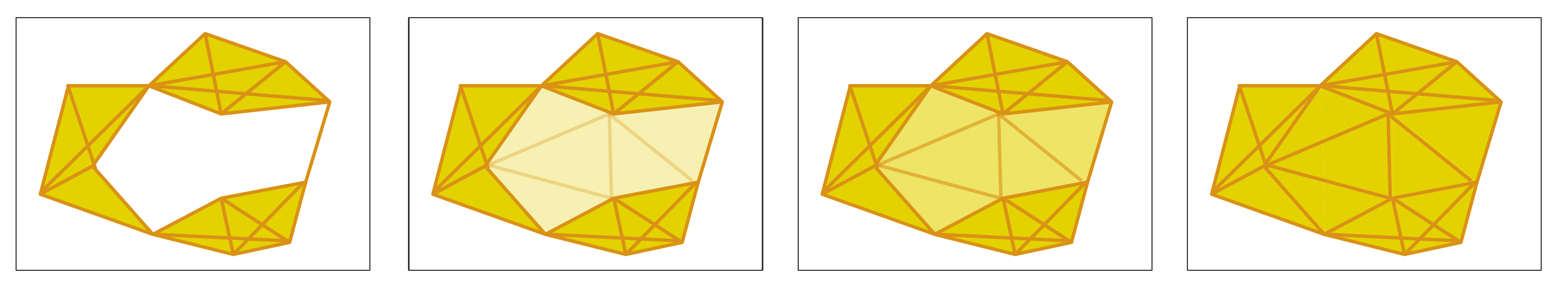}
		\caption{One can conceive of an analogous continuum between a simplicial complex with a hole ($1$-dimensional null-space of a weighted $\tilde{\mathcal{L}}_1$) and complex without holes ($0$-dimensional null-space of a weighted $\tilde{\mathcal{L}}_1$). The main technical contribution of the paper is the development of the theory of weighted higher-order/combinatorial Laplacian.}
		\label{fig:weighted-simplicial-complex}
	\end{subfigure}
%	\subfloat[\changedSB{One can conceive of an analogous continuum between a simplicial complex with a hole ($1$-dimensional null-space of a weighted $\tilde{\mathcal{L}}_1$) and complex without holes ($0$-dimensional null-space of a weighted $\tilde{\mathcal{L}}_1$). The main technical contribution of the paper is the development of the theory of weighted higher-order/combinatorial Laplacian.}]{\hspace{5em}{\includegraphics[width=0.75\columnwidth,trim=0 0 0 0,clip=true]{figures/weighted_simplicial_complex.pdf}\hspace{5em}} \label{fig:weighted-simplicial-complex}}
 
	\caption{\changedSB{Just as weighted graphs allow us to reason about and control clustering and weak connection between clusters through analysis of the spectrum and eigenvectors of $\tilde{\mathcal{L}}_0$ using gradient-based algorithms, a weighted simplicial complex would allow us to reason about and control large-scale holes bounding low-weight $2$-simplices through the analysis of the spectrum and eigenvectors ($2$-nd-order modes) of $\tilde{\mathcal{L}}_1$. The main technical contribution of the paper is the development of the theory of weighted higher-order/combinatorial Laplacian, and apply it to the control of sensor networks to fill and create holes in sensor coverage.}} \label{fig:graph-cplx-parallel}
	\vspace{-1em}
\end{figure} %{wrapfigure}

A natural extension of this definition to weighted graphs gives rise to the \emph{weighted graph Laplacian} which is defined as $\tilde{\mathcal{L}} = \tilde{{\mathcal B}} \tilde{{\mathcal B}}^\ast$, with the corresponding matrix representation $\tilde{L} = \tilde{B} \tilde{B}^T = B W^2 B^T$.
Real-valued weights on edges enable a continuum between graphs in which an edge is existing (with positive weight) and a graph in which the same edge is removed (the edge with zero weight) -- see Figure~\ref{fig:weighted-graph}. This continuum allows construction of gradient-based algorithms based on the spectrum of the weighted graph Laplacian in order to control the connectivity of the graph. Weighted graph Laplacian has hence been extensively used to control the connectivity of sensor networks~\citep{zavlanos2008distributed} and for transmission control in information networks~\citep{leiming-subhrajit}.

% \todo{need a figure to explain this, and transition to the more abstract definition.}

%}

\begin{figure}
\centering
\includegraphics[width=0.85\columnwidth,trim=0 0 0 0,clip=true]{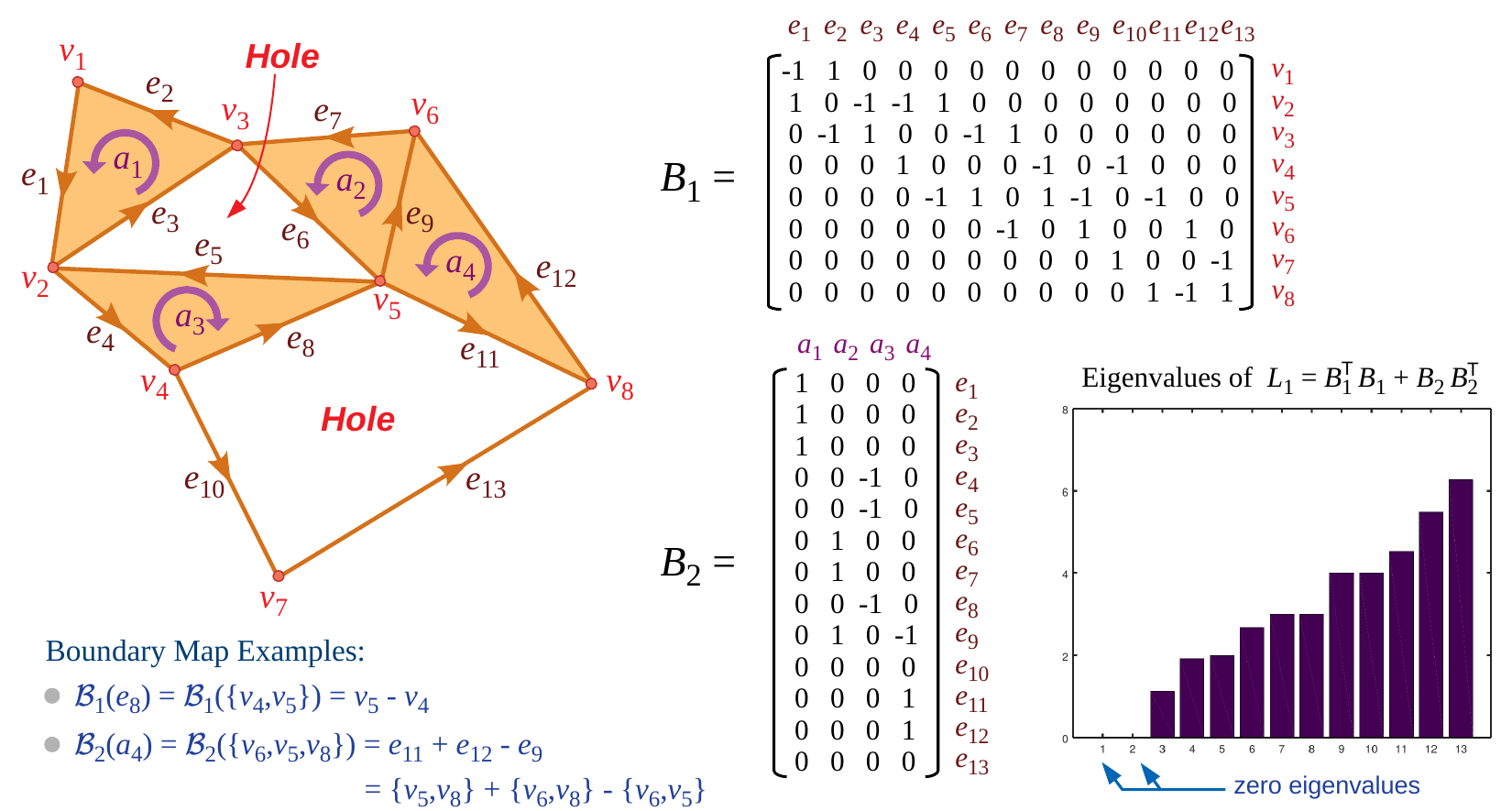}
	\caption{\changedSB{Example of a simplicial complex of dimension $2$ with arbitrary orientations assigned to the $1$ and $2$ simplices. The corresponding $1$-st and $2$-nd order (combinatorial, unweighted) boundary matrices are shown. Note how the higher order Laplacian matrix, $L_1$, has a null-space of dimension $2$ corresponding to the two holes (cycles that do not form the boundary of a set of $2$-simplices) marked in the complex.
 % \todo{SB: Add to figure: examples of ${\mathcal B}$ acting on simplices.}
 % The cycles made of $e_3,e_6,e_5$ and $e_8,e_{11},-e_{13},-e_{10}$ are the two holes in the shown simplicial complex.
 }}
	\label{fig:boundary-matrix-illustration}
\end{figure}

\subsection{Simplicial Complexes and Homology}

\changedSB{
% \emph{\textbf{Overview:}} 
A simplicial complex can be thought of as a higher-dimensional generalization of graphs, in which we not only can have the vertices (also called $0$-simplices) and edges (also called $1$-simplices, described by $2$-element subsets of $V$), but also $2$-simplices (can be thought of as ``\emph{triangle elements}'', described by $3$-element subsets of $V$) and higher-dimensional simplices.
Algebraic constructions using a similicial complex involves higher-dimensional extensions of the boundary map and Laplacian matrix, and allows us to reason about ``\emph{holes}'' in the complex, just as the graph Laplacian allowed reasoning about connected components of the graph.

The following technical discussions formalize the definition of simplicial complex, the boundary map, and introduces the concept of homology.}

% \todo{An introductory paragraph that describes simplicial complex as higher-dimensional extension of graphs and give a non-technical description of simplices of 2, 3 and higher dimensions.}

\begin{definition}[Simplicial Complex]
    For a finite set $V$, a simplicial complex $X$ is collection of subsets of $V$ that is closed under inclusion: That is, if $c \in X$
    % , then whenever $c^{\prime}\in X$ with 
    \changedSB{and}
    $c^{\prime} \subset c$, \changedSB{then} $c^{\prime} \in X$. Each element $c\in X$ is called $(|c|-1)$-simplex, where $|c|$ is the cardinality of $c$, and the subset of $X$ that contains all $n$-simplices is denoted $X_n$.  
\end{definition}
% Note that if $V$ is totally ordered, then $X$ is also totally ordered by the lexicographic order, and the order for both of them is denoted by $<$.
% \todo{A follow-up paragraph that describes different parts of the definition in lay-person's terms.}

%\vspace{0.5em}\noindent
\changedSB{In the definition above, $V$ is the vertex set, $X_0$ contains all the singletons of $V$ (the $0$-simplices), $X_1$ contains $2$-element subsets of $V$ (the $1$-simplices or edges), $X_2$ contains $3$-element subsets of $V$ (the $2$-simplices), and so on.
Given any $n$-simplex, $c\in X_n$, the condition of \emph{closure under inclusion} implies that the \emph{faces} of $c$ are also in the simplicial complex.
An explicit example is given below.}

\begin{example}
    1-simplices contain exactly two elements of $V$, and by definition, whenever a simplicial complex $X$ has a 1-simplex $\{x, y\} \subset V$ it must also contain $\{x\}, \{y\} \subset X$. This allows us to think of 0-simplices and 1-simplices (respectively) as vertices and edges (respectively,) where 1-simplices are bridging over 0-simplices. Hence, when a simplicial complex have 0-simplices and 1-simplices only, then it can be regarded as a simple graph. For example, \changedSB{the} graph, $G$, shown in Figure \ref{fig: a graph} can be written as 
    \begin{equation*}
        G = \{\{v_0\}, \{v_1\}, \{v_2\}, \{v_3\}, \{v_0, v_1\}, \{v_0, v_2\}, \{v_0, v_3\}, \{v_1, v_2\}, \{v_1, v_3\}, \{v_2, v_3\} \}.
    \end{equation*}
    \changedSB{whereas the simplicial complex, $H$, shown in Figure \ref{fig: a simplicial complex} can be written as 
    \begin{equation*}
        H = \{\{v_0\}, \{v_1\}, \{v_2\}, \{v_3\}, \{v_0, v_1\}, \{v_0, v_2\}, \{v_0, v_3\}, \{v_1, v_2\}, \{v_1, v_3\}, \{v_2, v_3\}, \changedSB{\{v_1,v_2,v_0\}} \}.
    \end{equation*}}
\end{example}

\begin{figure}
    \centering
    \begin{subfigure}[b]{0.3\textwidth}
    \centering
    \begin{tikzpicture}
        \node[below] at (1, 0.5) {\textbf{$v_0$}};
        \node[above] at (1, 1.732) {\textbf{$v_1$}};
        \node[left] at (0,-0.2) {\textbf{$v_2$}};
        \node[right] at (2, -0.2) {\textbf{$v_3$}};
        \draw[thick] (0,0) -- (2,0);
        \draw[thick] (0,0) -- (1,1.732);
        \draw[thick] (0,0) -- (1,0.577);
        \draw[thick] (1,0.577) -- (2,0);
        \draw[thick] (1,1.732) -- (2,0);
        \draw[thick] (1,0.577) -- (1,1.732);
    \end{tikzpicture}
    \caption{A Graph $G$}
    \label{fig: a graph}
    \end{subfigure}    
    \begin{subfigure}[b]{0.3\textwidth}
    \centering
    \begin{tikzpicture}
    \draw[gray, fill] (1,1.732) -- (1,0.577) -- (0,0);
    \node[below] at (1, 0.5) {\textbf{$v_0$}};
    \node[above] at (1, 1.732) {\textbf{$v_1$}};
    \node[left] at (0,-0.2) {\textbf{$v_2$}};
    \node[right] at (2, -0.2) {\textbf{$v_3$}};
    \draw[thick] (0,0) -- (2,0);
    \draw[thick] (0,0) -- (1,1.732);
    \draw[thick] (0,0) -- (1,0.577);
    \draw[thick] (1,0.577) -- (2,0);
    \draw[thick] (1,1.732) -- (2,0);
    \draw[thick] (1,0.577) -- (1,1.732);
    \end{tikzpicture}
    \caption{A Simplicial Complex $H$}
    \label{fig: a simplicial complex}
    \end{subfigure}
    \caption{Example of a graph and a simplicial complex}
\end{figure}
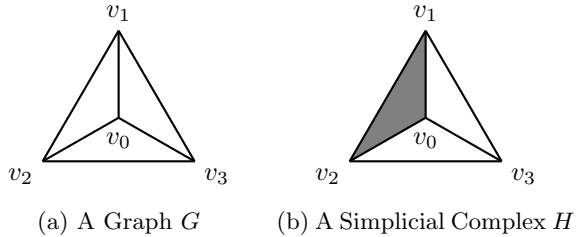

Without loss of generality we will assume that the union of 0-simplices is equal to $V$, and thus we refer to call $V$ as a \textit{vertex set}.
\begin{example}
    More generally, for $n\geq 0$, an $n$-simplex can be regarded as an $n$-dimensional (solid) triangles. A 2-simplex $\{a, b, c\} \subset X$ can be seen as a normal triangle with three vertices $\{a\}, \{b\}$ and $\{c\}$, and similarly a 3-simplex $\{a, b, c, d\} \subset X$ can be seen as a tetrahedron with three vertices $\{a\}, \{b\}, \{c\}$ and $\{d\}$.
\end{example}
% Our goal of the first half of this section is to compute the holes in the complex via homology groups. 

\changedSB{The key algebraic description of a simplicial complex is encapsulated in the definition of the higher-order (combinatorial) boundary maps.
In this generalization of the boundary map, the vector space spanned by the $0$-simplices is denoted by $C_0$, vector space spanned by the $1$-simplices (edges) is denoted by $C_1$, vector space spanned by the $2$-simplices is denoted by $C_2$, and in general, vector space spanned by the $n$-simplices is denoted by $C_n$.
Upon assigning an arbitrary orientation to each simplex,
the boundary maps are %between $C_n$ and $C_{n-1}$ is 
denoted by ${\mathcal B}_n:C_n \rightarrow C_{n-1}$, whose matrix representations (in the basis of the simplices) are $B_n$ (see Figure~\ref{fig:boundary-matrix-illustration} for an explicit example of $B_1$ and $B_2$). The technical discussions below gives the formal definitions and highlights the key property of the boundary maps, which paves the way to eason about holes in the simplicial complex.}

% In order to \changedSB{reason about holes in the simplicial complex}, 
We associate \changedSB{with} a simplicial complex a natural algebraic structure that encodes the boundary information, called \textit{Chain Complex}. 

\begin{definition}[Chain Complex and Boundary Maps over Reals]
    A sequence of vector spaces $\{C_n\}_{n\geq 0}$ over $\R$ equipped with the set of linear maps ${\mathcal B}_n: C_n \to C_{n-1}$ satisfying ${\mathcal B}_{n-1} \circ {\mathcal B}_{n} = 0$ for each $n>0$ is called a chain complex and the maps ${\mathcal B}_n$ are called boundary maps.
\end{definition}

There is a natural way of assigning an algebraic structure of this kind to simplicial complexes, which is stated in the next proposition. We omit the proof, but interested readers can find the proof in a standard textbook of algebraic topology (See, for example, \citep{hatcher}.)
\begin{proposition}[Simplicial Chain Complex~\changedSBa{\citep{hatcher}}]
    Let $X$ be a simplicial complex over a vertex set $V$, and let $C_n\,X$ be the vector spaces over $\R$ with basis the $n$-simplices of $X$ and let ${\mathcal B}_{\changedSB{n}}: C_n \to C_{n-1}$ to be the linear map obtained by linearly extending the following formula:
    \begin{equation*}
        {\mathcal B}_{\changedSB{n}} (\{v_0, v_1, \cdots, v_n \}) = \sum_{i=0}^n (-1)^i \{v_0, v_1, \cdots, v_n \} \setminus \{v_i\}.
    \end{equation*}
    Then $C\, X:=\{C_n\}_{n \geq 0}$ is a chain complex.
\end{proposition}
\noindent \changedSB{The definition of the boundary maps in the above proposition assigns a \emph{default} orientation to the simplices based on a chosen ordering of the vertices in $V$.}

Having defined the algebraic structure on the simplicial complex based on the boundary information, the holes can now be measured by how much the boundary information between adjacent dimensions fails to match up: Let $G$ be the graph defined on the vertex set $V=\{v_0,\cdots, v_3\}$ as shown in the figure \ref{fig: a graph}. Then we intuitively see that $G$ has three 1-dimensional holes, that is, the holes made up by 1-simplices. If we take the formal sum of edges creating a hole preserving the orientation, we can write down these three holes as 1-chains:
\begin{align*}
    c_1 = & \{v_0, v_1\} - \{v_0, v_2\} + \{v_1, v_2\}, \\
    c_2 = & \{v_0, v_1\} - \{v_0, v_3\} + \{v_1, v_3\}, \\
    c_3 = & \{v_0, v_2\} - \{v_0, v_3\} + \{v_2, v_3\}.
\end{align*}
The important properties of these chains are that they will vanish when we take a boundary map; for example, 
\begin{align*}
    {\mathcal B}_1 c_1 = {\mathcal B}_1( \{v_1, v_2\} - \{v_1, v_3\} + \{v_2, v_3\}) = ( \{v_1\} - \{v_2\}) - (\{v_1\}-\{v_3\}) + (\{v_2\}-\{v_3\}) = 0
\end{align*}
This property thus motivates us to define the set of 1-dimensional holes of $G$ to be the nullspace of ${\mathcal B}_1$, denoted $\Ker {\mathcal B}$. This may seem sufficient for the definition of the holes on a graph, but it is in fact not sufficient for the higher dimensional simplicial complex. Let $H$ be the simplicial complex shown in the figure \ref{fig: a simplicial complex}, which is just $G$ but with a 2-simplex attached on it. Now one of the three hole that were present in $G$, namely $c1$, is no longer a hole since it is now covered by a 2-simplex. To realize this we compute the image of the 2-simplex $f$ under ${\mathcal B}$
\begin{align*}
    {\mathcal B}_2 f = \{v_1, v_2\} - \{v_1, v_3\} + \{v_2, v_3\},
\end{align*}
and notice that this equals to $c_1$. Hence, in order to take the higher dimensional simplices covering up the holes into account, we "forget" all of the images of ${\mathcal B}_2$ from $\Ker {\mathcal B}_1$, and this is exactly the the definition of the first homology, and we can more generally define:
\begin{definition}[Homology with Real Coefficients]\label{homology}
    Let $\{C_n\}$ be a chain complex. Then $n$-th homology $H_n\,C$ with real coefficients are the sequence of vector spaces defined by
    \begin{equation*}
        H_n\,C = \Ker {\mathcal B}_n / \Img {\mathcal B}_{n+1}. 
    \end{equation*}
\end{definition}
The intuition remains the same for the higher dimensional homology. The vectors in $\Ker {\mathcal B}_n$ represent the holes made up by $n$-simplices, and the vectors in $\Img {\mathcal B}_{n+1}$ are those holes that are filled up by the $(n+1)$-simplices, which we "forget" by taking quotient. The resulting vector space $H_n\,C$ thus contains the actual $n$-dimensional holes in $C$.

\subsection{Combinatorial Laplacian}
%\todo{need some citation in this section.}
We now turn to an overview of what is called combinatorial Laplacian. Let ${\mathcal B}_{\changedSB{n}}^\ast: C_{n-1}\to C_n$ be the \textit{adjoint} of ${\mathcal B}$ \changedSB{with respect to an inner product $\langle\cdot,\cdot\rangle$} (i.e., a linear map that satisfies
% \begin{equation*}
   $ \langle u, {\mathcal B}_n v \rangle = \langle {\mathcal B}^\ast_n u, v \rangle $
% \end{equation*}
for each $u\in C_{n-1}$ and $v\in C_n$). Then the \textit{$n$-th combinatorial Laplacian} $\mathcal L_n: C_n \to C_n$ is defined by
\begin{align*}
    \mathcal L_n %(v) 
    = {\mathcal B}_{\changedSB{n}}^\changedSB{\ast} {\mathcal B}_{\changedSB{n}} + {\mathcal B}_{\changedSB{n+1}} {\mathcal B}_{\changedSB{n+1}}^\ast. %(v).
\end{align*}
% for each $v\in C_n$.
It is often convenient to define
\begin{align*}
    \mathcal L_n^{\textit{down}} 
    = {\mathcal B}_n^\ast {\mathcal B}_n \quad \text{ and } \quad \mathcal L_n^{\textit{up}} 
    = {\mathcal B}_{n+1} {\mathcal B}_{n+1}^\ast.
\end{align*}
Note that we have $\mathcal L_n = \mathcal L_n^{\textit{down}} + \mathcal L_n^{\textit{up}}$ by definition. The

\changedSB{Using the simplices as basis and the standard Euclidean norm on the vector spaces, the matrix representation of the combinatorial Laplacian is given by
\[L_n = B_n^{T} B_n + B_{n+1} B_{n+1}^{T}\]
The graph Laplacian matrix described earlier is then $L_0$ (assuming ${\mathcal B}_0 = 0$).
}

% \todo{mention subscripts of boundary maps explicitly, check for the order of the product of boundary and its adjoint.}

It is clear that $\mathcal L_{\changedSB{n}}$ is positive semi-definite, and in particular all of their eigenvalues are real and non-negative. Moreover, the self-adjointness used in the construction, as well as the fact that ${\mathcal B}_n$ is a boundary map, carries out several important properties of the Combinatorial Laplacian.
\begin{proposition}\label{fundamentals of laplacian}
    For each $n \geq 0$,
    \begin{enumerate}[label=(\alph*),leftmargin=2em]
        \item $\Ker \mathcal L^{\textit{down}}_n = \Ker {\mathcal B}_n$ and $\Ker \mathcal L^{\textit{up}}_n = \Ker {\mathcal B}_{n+1}^\ast$,
        \item $\Ker \mathcal L_n = \Ker (\mathcal L^{\textit{down}}_n + \mathcal L^{\textit{up}}_n ) = \Ker \mathcal L^{\textit{down}}_n \cap \Ker \mathcal L^{\textit{up}}_n$, and
        \item $\Img \mathcal L_n^{\textit{up}} \subseteq \Ker \mathcal L_n^{\textit{down}}$ and $\Img \mathcal L_n^{\textit{down}} \subseteq \Ker \mathcal L_n^{\textit{up}}$.
    \end{enumerate}
\end{proposition}

\vspace{0.5em}\noindent
In particular, if $\lambda$ is an eigenvalue of $\mathcal L_n$ it is either an eigenvalue of $\mathcal L_n^{\textit{down}}$ or $\mathcal L_n^{\textit{up}}$, and $\lambda=0$ if and only if it is an eigenvalue of both $\mathcal L_n^{\textit{down}}$ and $\mathcal L_n^{\textit{up}}$. This fact can rephrased by using the homology vector spaces:
\begin{theorem}[Discrete Hodge Theorem~\changedSBa{\citep{arnold2010finite,derenick}}] \label{Discrete Hodge Theorem}
\begin{equation}
    \Ker \mathcal L_n \cong H_n C.
\end{equation}
\end{theorem}
The dimension of the kernel of the combinatorial Laplacian $\mathcal L_n$ thus counts the number of holes in the complex $C$. 

\section{Weighted Laplacian}\label{sec:weighted-laplacian}
In this section we \changedSB{provide the main theoretical contribution of this paper.

\vspace{0.5em}\noindent
\textbf{Summary of the Technical Contribution of this Section:}
In summary, we assign positive weights to each simplex in the complex and define the $n$-th \changedSBa{diagonal} weight matrix, $W_n \in \mathbb{R}^{|X_n|\times |X_n|}$, such that the $(i,i)^\text{th}$ diagonal element is the weight on the $i^\text{th}$ $n$-simplex. Using the basis of the simplices, the definition of the $n$-th weighted boundary matrix is
\[ \tilde{B}_n = W_{n-1}^{-1} B_{n} W_n\]
which leads to the definition of the $n$-th weighted combinatorial Laplacian matrix,
\begin{eqnarray*}
\tilde{L}_n 
& = & \tilde{B}_n^T \tilde{B}_n + \tilde{B}_{n+1} \tilde{B}_{n+1}^T \\
& = &  W_n B_n^T (W_{n-1}^{-1})^2 B_n W_n + W_n^{-1} B_{n+1} W_{n+1}^2 B_{n+1}^T W_n^{-1}.
\end{eqnarray*}
%

%It is worth noting that with 
\noindent
Furthermore, we assume that the weights are chosen such that the weight on a face of a simplex is always greater than the weight on the simplex itself (referred to as the \emph{filtration condition}).
With this restriction, it can be seen that in the definition of the weighted boundary matrix, even if we let the weight on a $(n-1)$-simplex to get infinitesimally small, since the weight on any $n$-simplex whose face it constitutes is even smaller, the weighted boundary matrix remains finite.

% our weighted combinatorial Laplacian satisfies the following properties:
With the filtration condition, the main theoretical result of this section can be summarized as follows:
\begin{quote}
%    \item[i.] %In Lemma~\ref{norm of a union}, we show that the 2-norm of the weighted Laplacian is bounded from above by a value not depending on the choice of weights.
\normalsize
    \emph{In Theorem~\ref{main theorem} we show that in a weighted simplicial complex, even if there does not exist a hole (and hence the kernel of the weighted Laplacian is trivial), a subcomplex with low weights (corresponding to an ``\emph{almost hole}'') results in a \emph{small} eigenvalue of the Laplacian with an upper bound proportional to the low weight on the subcomplex.
    This result is further generalized to a complex with multiple \emph{almost-holes} in Proposition~\ref{extended main theorem}.}

%    \item[ii.] In Theorem~\ref{main theorem}, we show that the smallest eigenvalues tend to emerge around the simplices of smaller weights.
\end{quote}
\changedSBa{The proof of all the theoretical results (lemmas, propositions and theorems) presented in this section are placed in Appendix~\ref{appendix:proofs} for better readability.}

%\todo{Shunsaku: Please check the above.}

} 

%\todo{Need to clean up some of the notations and definitions. Let's choose weights to be strictly positive from the very beginning. \changedSY{Weights are changed to be strictly positive.}}

\subsection{The Definition and the Basic Properties}
\begin{definition}[Weights on Simplicial Complex]
    Let $X$ be a simplicial complex. A function $w_n: X_n\to \R_{>0}$ is called a \textbf{weight function}, and $w_n$ can always be linearly extended to the linear map $C_n(X) \to C_n(X)$, which we will denote by $\mathcal W_n$. The tuple $\Tilde X := (X, \;w)$ is called a \textbf{weighted simplicial complex}. The subscript $n$ of $w_n$ and $\mathcal W_n$ is often abbreviated when there is no confusion.
\end{definition}
Note that, when written out as a matrix, $\mathcal W_n$ is the diagonal matrix with the values of $w_n$ on its diagonal.
% \todo{SB: I moved the definitions of intersection and union to a later part, right before where it's used.}

% \todo{Shunsaku: Use $\mathcal W$ for the weight maps, and $W$ for its matrix representation in the simplicial basis. Also consider changing the boundary maps to $\mathcal{B}$ from $\partial$. \changedSY{Changed the boundary maps from $\partial$ to $\mathcal{B}$. Changed the weight maps from $W$ to $\mathcal{W}$.  }}

\begin{definition}[Weighted Chain Complex and Weighted Laplacian]
    Let $\Tilde X = (X, w)$ be a weighted simplicial complex. Define the sequence of vector spaces $\Tilde C(\Tilde X) := C(X)$. We define our \textbf{weighted chain complex} structure of $\Tilde C(\Tilde X)$ by providing the \textbf{weighted boundary maps} $\Tilde {\mathcal B}_n$ given by
    \begin{align*}
        \Tilde{{\mathcal B}}_n = \mathcal W_{n-1}^{-1} {\mathcal B}_n \mathcal W_n
    \end{align*}
    and the \textbf{weighted combinatorial Laplacian}
    \begin{align*}
        \Tilde{\mathcal L_n} = \Tilde {\mathcal B}_{n}^{\ast} \Tilde {\mathcal B}_n + \Tilde {\mathcal B}_{n+1} \Tilde {\mathcal B}_{n+1}^{\ast}.
    \end{align*}
\end{definition}
It is straightforward to verify that the maps $\Tilde {\mathcal B}_n$ are in fact boundary maps. Hence we can systematically define the homology vector spaces for weighted chain complexes as well, which we will denote by $H_n \, \Tilde X$. From the definition, it is clear that the diagram
\begin{center}
\begin{tikzcd}
    \cdots \arrow{r} & \Tilde C_{n+1} \arrow{r}{\Tilde {\mathcal B}_{n+1}} \arrow{d}{W_{n+1}} & \Tilde C_n \arrow{r}{\Tilde {\mathcal B}_n} \arrow{d}{W_{n}} & \Tilde C_{n-1} \arrow{r} \arrow{d}{W_{n-1}} & \cdots \\
    \cdots \arrow{r} & C_{n+1} \arrow{r}{{\mathcal B}_{n+1}} & C_n \arrow{r}{{\mathcal B}_n} & C_{n-1} \arrow{r} & \cdots.
\end{tikzcd}
\end{center}
commutes, that is, ${\mathcal B}_{n} \mathcal W_{n} = \mathcal W_{n-1} \Tilde {\mathcal B}_{n}$ for each $n\geq 0$. Since $w_n$ is positive, each $\mathcal W_n$ is an isomorphism, thus we obtain
\begin{equation}\label{eq:weighted_h_isom}
    H_n \, \Tilde X \cong H_n\, X.
\end{equation}
Furthermore, recall that the results in the previous section about the combinatorial Laplacian did not depend on anything beyond the self-adjointness and the fact that ${\mathcal B}_n$ is a boundary map. Hence, Theorem \ref{Discrete Hodge Theorem} applies to our weighted homology and weighted Laplacian as well, which in turn implies
\begin{equation}\label{eq:weighted_l_ker_isom}
    \Ker(\Tilde{ \mathcal L_n}) \cong \Ker( \mathcal L_n),
\end{equation}
where $\Tilde{ \mathcal L_n}$ is the weighted combinatorial Laplacian defined over $\Tilde{X} = (X, w_n)$ and $\mathcal L_n$ is the combinatorial Laplacian defined over $X$.

As with the case of combinatorial Laplacian, we can write our weighted combinatorial Laplacian explicitly; if $v = \alpha_1 c_1 + \cdots + \alpha_{|X_n|} c_{|X_n|} \in  \Tilde{C}_n\, X$, then
\begin{equation}
    \langle \Tilde{\mathcal L}v, c_i \rangle  = \sum_{f\in X_{n+1}: \;c_i \subseteq f} \sigma(c_i, f) \sum_{c_j \in X_n:\; c_j \subseteq f} \sigma(c_j, f) \frac{w(f)^2}{w(c_i)w(c_j)} \alpha_j,
\end{equation}
where $\sigma(c, f)$ is the sign function determined by the sign of $c$ in the sum ${\mathcal B} f$, that is,
\begin{align*}
    \sigma(c, f) = \langle c, {\mathcal B} f \rangle.
\end{align*}

% Finally, we make it clear what happens when the weights of some cells are zeros:
% \begin{proposition}\label{zero weights}
%     Let $C \subset D$ be a subcomplex of $C$. Suppose further that the weight of the cell of $D$ not in $C$ is zero. Then the following diagram
%     \begin{center}
%     \begin{tikzcd}
%     C \arrow{r}{\iota} \arrow{d}{\Tilde{\mathcal L}_C} & D \arrow{d}{\Tilde{\mathcal L}_D} \\
%     C \arrow{r}{\iota} & D
%     \end{tikzcd}
%     \end{center}
%     commutes, where $\iota$ is the obvious inclusion map, and the $\Tilde{\mathcal L}_C$ and $\Tilde{\mathcal L}_D$ are the weighted Laplacian operators of $C$ and $D$, respectively. 
% \end{proposition}

% The proposition above tells us that adding a simplex of weight zero to the weighted complex does not change the nonzero eigenvalues and their corresponding eigenspaces of the weighted Laplacian but only the kernel of it. indeed, if $v\in C^{\prime}$ is an eigenvector of $L^{\prime}$ such that $L^{\prime} v = \lambda v$, then by the commutativity of the diagram we have $L \iota v = \iota L^\prime v = \lambda \iota v$, hence $\iota v \in C$ is also an eigenvector of $L$ corresponding to the eigenvalue  $\lambda$.

\subsection{Norm}
We first consider the norm of our weighted Laplacian matrix. If we choose the weight function $\{w_n\}_{n\geq 0}$ completely independently, then $\| \Tilde{ \mathcal L}_n \|$ is unbounded; for example, let $X$ be a simple graph with one edge and, choose the constant weight functions $w_0=1$ and $w_1=m$ for some positive number $m$. Then one can compute that $\| \Tilde{\mathcal L}_0 \| = 2m$, which clearly diverges as $m\to \infty$. Fortunately, there is a natural local condition on weights which gives us an upper bound of the 2-norm of the weighted Laplacian that is independent of the choice of weights:

\begin{definition}
    We say that the weight functions $w:X\to \R_{\geq 0}$ satisfies the \textbf{filtration condition} if for each $n \geq 1$ and for each simplex $b\in X_{n-1}$ and $c\in X_n$ with $b \subseteq c$, we have
    \begin{align*}
        w_n(c) \leq w_{n-1}(b).
    \end{align*}
\end{definition}

\begin{proposition}\label{Norm}
    If the weights $\{w_n\}$ of a weighted simplicial complex $\Tilde{X_n}$ satisfy the filtration condition, then
    \begin{align*}
        \| \Tilde{ \mathcal L }_n \| \leq (n+2)|X_n|.
    \end{align*}
\end{proposition}

The filtration condition is not only useful for bounding the norm from above, but it also allows us bound the small eigenvalues from above, as we shall see.

\subsection{Small Eigenvalues}
We now investigate the most important property of our weighted Laplacian; the behavior of the small eigenvalues and their implication. Our approach is to decompose a weighted simplicial complex into a union of multiple smaller complexes with larger weights and smaller weights, and observe that the small eigenvalues of the original weighted complex tends to emerge around the cells with smaller weights. The main results are stated in Theorem \ref{main theorem} and its direct generalization Proposition \ref{extended main theorem}.

\changedSB{In the following discussion}, by the intersection $\Tilde X \cap \Tilde Y$ (or the union $\Tilde X \cup \Tilde Y$, respectively) of weighted simplicial complex $\Tilde X=(X,\;w_X)$ and $\Tilde Y=(Y,\;w_Y)$, we mean the tuple $(X \cap Y,\;w_{X \cap Y})$ (or $(X \cup Y,\;w_{X \cup Y})$, respectively) where 
\begin{align*}
    w_{X \cap Y}(c) &= w_{X}(c), \text{ and } \\
    w_{X \cup Y}(c) &= \begin{cases}
        w_{X}(c) & \text{ if $c\in X$} \\
        w_{Y}(c) & \text{ if $c\in Y$}
    \end{cases} 
\end{align*}
Hence, for the union and the intersection of weighted simplicial complexes to be well-defined, whenever we write $\Tilde X \cap \Tilde Y$ or $\Tilde X \cup \Tilde Y$, it is understood that $X$ and $Y$ are defined on the same vertex set and their weight functions satisfy $w_X(c) = w_Y(c)$ for each $c \in X \cap Y$. 
% \todo{Replace this by restriction of $w$ to a subset.}

Let us consider the following motivating example. Let $\Tilde X=(X, w_X)$ be a weighted simplicial complex illustrated in Figure \ref{fig:union_example_x}, and let us further assume that the weight function $w_X$ is constant, say $w_X=1$. Combinatorially, $X$ admits one 1-dimensional hole, and from the isomorphism (1), the weighted complex $\Tilde{X}$ as well admits a 1-dimensional hole. We now introduce another weighted simplicial complex $\Tilde Y=(Y,w_Y)$, as shown in Figure \ref{fig:union_example_y}, and suppose that each simplex of $\Tilde Y - \Tilde X$ has weight $\epsilon$, for some small $\epsilon>0$. By identifying certain boundaries, we can form a union $\Tilde X \cup \Tilde Y$, and fill the hole of $\Tilde X$ with a sheet $\Tilde Y$ with small weights. In fact, $\Tilde X \cup \Tilde Y$ has no combinatorially-defined 1-dimensional hole in it. However, because weights of $\Tilde Y$ are so small, we would like to show that the smallest eigenvalue of $\Tilde X \cup \Tilde Y$ is also small so that it almost belongs to the kernel of $\Tilde {\mathcal L}$, giving the notion of ``almost 1-dimensional hole'' mentioned in the introduction section. Furthermore, we would like to show the corresponding eigenvector also emerge around $\Tilde Y$, the cells of small weights. This behavior of weighted complex is formally stated in Theorem \ref{main theorem}, and our goal of this section is to prove it.

\begin{figure}
\centering
\begin{subfigure}[b]{0.3\textwidth}
    \includegraphics[width=\textwidth]{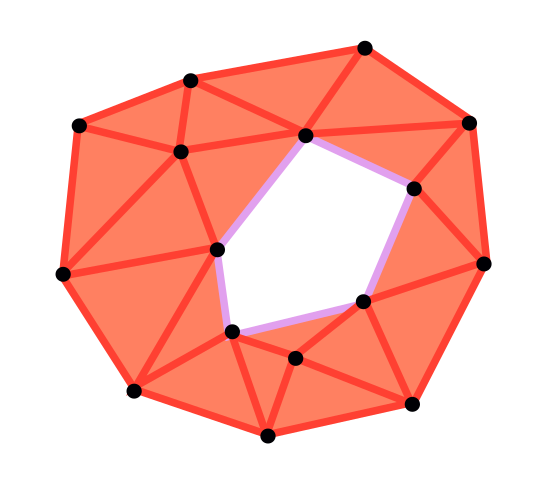} 
    \caption{$\Tilde X$}
    \label{fig:union_example_x}
\end{subfigure}
\hfill
\begin{subfigure}[b]{0.3\textwidth}
    \centering
    \includegraphics[width=0.5\textwidth]{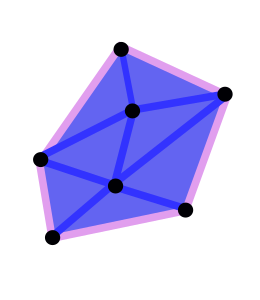}
    \caption{$\Tilde Y$}
    \label{fig:union_example_y}
\end{subfigure}
\hfill
\begin{subfigure}[b]{0.3\textwidth}
    \includegraphics[width=\textwidth]{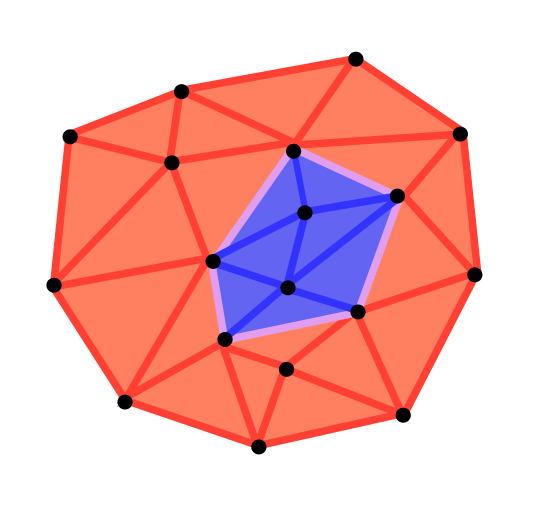}
    \caption{$\Tilde X \cup \Tilde Y$}
    \label{fig:union_example_x_union_y}
\end{subfigure}
\caption{Example of a union}
\label{fig:union_example}
\end{figure}

\begin{definition}
    We define the topological boundary operator $\partial_n X$ of a simplicial complex $X$ at dimension $n$ to be
    \begin{align*}
        \partial_n X = \{ a \in X_n : \text{There is exactly one $b\in X_{n+1}$ such that $a\cap b = a$} \}.
    \end{align*}
\end{definition}
% \todo{This definition seems a little too restrictive. Do you need $X \cup Y$ to be closed (without boundary)? For future iterations, let's try to see if we can relax this.}
% \changedSY{This definition is used in the main theorem (Theorem \ref{main theorem}) to work. The theorem requires that the intersection of two complex $\Tilde{X}$ and $\Tilde{Y}$ be "small", and this is accomplished by requiring $\Tilde{Y}$ be glued along $\partial_n \Tilde Y$. There would be a lot of counterexamples for the main theorem with  $\Tilde{X} \cap \Tilde{Y}$ not small enough.}

% Readers are warned not to confuse $\partial_n$ with the boundary map ${\mathcal B}_n$: $\partial_n$ is the operator on a simplicial complex $X$ that spits out another simplicial complex $\partial_n X$, while the boundary map ${\mathcal B}_n$ is a linear map defined on a simplicial chain complex.
% \todo{Suggestion for Shunsaku: Use $\mathcal{B}$ for boundary operator thought the paper instead of $\partial$.} \changedSY{Changed to $\mathcal B$}

We start with a basic property of the weighted Laplacian with respect to union: 
\begin{lemma}\label{norm of a union}
    If $\Tilde X$ and $\Tilde Y$ are weighted simplicial complex such that $\Tilde X \cap \Tilde Y = \partial_n \Tilde Y$ and if $v \in C(\Tilde X \cup \Tilde Y)$, then 
    \begin{align*}
        \| \mathcal L^{\textit{up}}_{\Tilde X \cup \Tilde Y} v\| \leq \| \mathcal L^{\textit{up}}_{\Tilde X} \pi_{\Tilde{X}}(v) + \mathcal L^{\textit{up}}_{\Tilde Y} \pi_{\Tilde Y} (v) \|.
    \end{align*}
    where $\pi_{\Tilde X}: C(\Tilde X \cup \Tilde Y) \to C \Tilde X$ and $\pi_{\Tilde Y}: C(\Tilde X \cup \Tilde Y) \to C \Tilde Y$ are the obvious projection maps.
\end{lemma}

% \todo{
% Looking at Figure 6, shouldn't you be saying $\Tilde X \cap \Tilde Y = \partial_n \Tilde Y$? Or exchange $X$ and $Y$ in the figure?
% We need to mention what $v$ is in the statement of the above Lemma. Also, $\mathcal{L}^\text{up}$ has not ben defined before.
% \changedSY{Yes, now it is changed to $\Tilde X \cap \Tilde Y = \partial_n \Tilde Y$. Description of $v$ is also added.}}

\begin{lemma}\label{inequality about a union}
    Let $X$ be a simplicial complex, and let $v$ be a chain in $C_n(\partial_n X) \subseteq C_n\, X$. Then 
    \begin{align*}
        \| \mathcal L^{\textit{up}} v \| \leq (n+1)\| v \|.
    \end{align*}
\end{lemma}

\begin{theorem}\label{main theorem}
    Let $\epsilon>0$, and let $\Tilde X = (X, w_X)$ and $\Tilde Y=(Y,w_Y)$ be weighted simplicial chain complexes that satisfy
    \begin{enumerate}[leftmargin=2em]
        \item \changedSY{$w_X$ and $w_Y$ satisfy the filtration condition,}
        \item $H_n \Tilde Y = 0$,
        \item $\dim H_n \left(\Tilde X \cup \Tilde Y \right) = \dim \left( H_n \Tilde X \right) - 1$,
        \item $\Tilde X \cap \Tilde Y = \partial_n \; \Tilde Y$, and 
        \item $w_X(a)/w_X(b) \leq \epsilon$ for all $a\in \Tilde X$ and $b\in \Tilde X \cap \Tilde Y$.
    \end{enumerate}
    Then there is a unit vector $v\in \Tilde C_n \Tilde X$ such that 
    \begin{equation}\label{upper bound for the smallest eigenvalue}
        \left \| \Tilde {\mathcal L} _{\Tilde X \cup \Tilde Y} (v) \right\| \leq \epsilon (n+1)
    \end{equation}
    In particular, the smallest nonzero eigenvalue of $\Tilde {\mathcal L} _{\Tilde X \cup \Tilde Y, n}$ is bounded by $\epsilon (n+1)$.
\end{theorem}

% In the view of Proposition \ref{zero weights}, it is clear that each eigenvalue of $\Tilde{ \mathcal L }_{\Tilde X \cup \Tilde Y}$ converges to some eigenvalue of $\Tilde{ \mathcal L }_{\Tilde X}$ as $\epsilon$ tends to 0, due to the continuity of the eigenvalues of $\Tilde{ \mathcal L }$ with respect to the weights. Hence if $\epsilon$ is sufficiently small, then the inequality (\ref{upper bound for the smallest eigenvalue}) in the previous theorem is indeed an upper bound for the smallest eigenvalue of $\Tilde{ \mathcal L }_{\Tilde X \cup \Tilde Y}$.

\vspace{0.5em}\noindent
\changedSBa{The interpretation of the above Theorem is that in a weighted simplicial complex, $\Tilde{X}\cup\Tilde{Y}$, even if there does not exist a hole (and hence no $1$-cycle in the kernel of the weighted Laplacian), if the simplices on a sub-complex, $\Tilde{Y}$ (and hence $\Tilde{X}\cap\Tilde{Y}$), have low weights compared to the simplices on the the rest of the complex (by a factor of $\epsilon$, hence creating an \emph{almost-hole}), then the spectrum of the weighted Laplacian will have a low eigenvalue that is bounded by $\epsilon(1+n)$.}

\changedSBa{With a little extra} effort \changedSBa{it is possible} to extend the Theorem \ref{main theorem} to a sequence of successive union operations \changedSBa{thus extending the result to weighted complexes with multiple almost-holes}:
 
\begin{proposition}\label{extended main theorem}
    Let $\Tilde X$ be a weighted simplicial complex. For each natural number $k \geq 1$ and for each $1 \leq i \leq k$, define $\epsilon_i>0$ and a weighted simplicial complex $\Tilde Y_i$. Assume that each triple $(\Tilde X \cup \bigcup_{j=1}^{i-1} \Tilde Y_j,  \Tilde Y_i, \epsilon_i)$ satisfies the conditions (1)-(5) of Theorem \ref{main theorem} and that $\Tilde Y_i$'s are pairwise disjoint.
    Then for each $1 \leq i \leq k$ there is a unit vector $v_i\in \Tilde C_n \Tilde X$ such that 
    \begin{equation}\label{successive upper bound}
        \left \| \Tilde {\mathcal L}_{\Tilde X \cup \bigcup_{j=1}^{k} \Tilde Y_j} (v_i) \right\| \leq \epsilon_i (n+1),
    \end{equation}
    and for each $1\leq i < j \leq k$, 
    \begin{equation}\label{orthogonal condition for v_i}
        \langle v_i, v_j \rangle = 0.
    \end{equation}
    Therefore, if each $\epsilon_i$ is sufficiently small, then (\ref{successive upper bound}) is an upper bound for distinct eigenvalues of $\Tilde {\mathcal L}_{\Tilde X \cup \bigcup_{j=1}^{k} \Tilde Y_j}$.
\end{proposition}

\section{\changedSB{Results and Applications}} %{Dynamic Coverage Repair of Sensor Network}
\label{sec:application}

\changedSB{As an initial demonstration of the proposed weighted combinatorial Laplacian, we refer to the result presented in Figure~\ref{fig:example of l1}, where it can be seen that, despite a simplicial complex not having holes, regions of low-weight simplices (the interior of the three circular rings around which most of the sensors are placed) can be identified by the small eigenvalues in the spectrum of $\tilde{\mathcal{L}}_1$.}

Figure~\ref{static network} shows the computational results for a sequence of \changedSBa{weighted} simplicial complexes, where the weights of some of simplices gradually decrease. Each of the five simplicial complexes \changedSBa{have} no combinatorial holes defined on it, but the two smallest eigenvalues are getting smaller and smaller as the two ``almost holes'' get significant.

% We now turn to the application of the weighted Laplacian.
\begin{figure}
\centering
\includegraphics[scale=0.25]{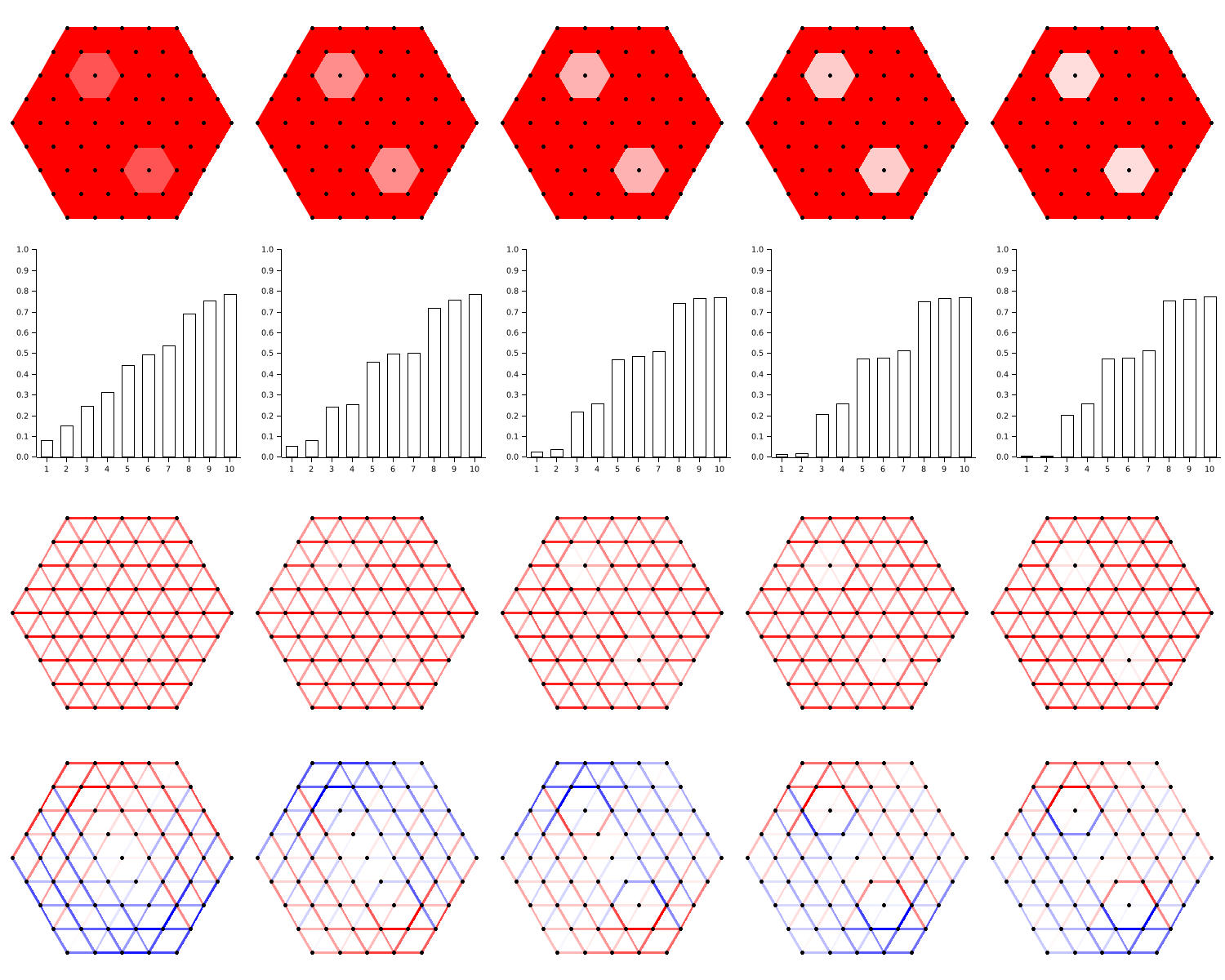}
\caption{The top row shows the five simplicial complexes whose weights on the 1- and 2-simplices are depicted in color: The color red means weight 1 and the color white means weight 0. The second row shows the first ten eigenvalues of the corresponding weighted Laplacian. The third row (resp. the fourth row) shows the coloring of edges by the eigenvector of the first (resp. the second) eigenvalue. }
\label{static network}
\end{figure}

\subsection{Dynamic Coverage Repair of Sensor Network}\label{Coverage Repair Method}
Simplicial complexes are proved to be useful in studying coverage problems of the sensor network. Given $n$ sensors $\mathcal S = \{s_1, \cdots, s_n\}$ in $\R^2$ and $\epsilon>0$, the simplicial complex $\mathcal R(\epsilon)$ defined by
\begin{align*}
    \mathcal R(\epsilon) = \{c \subset \mathcal S | \text{ for each $a, b \in c$, $d(a,b)<\epsilon$ }\}
\end{align*} 
is called a \textit{Vietoris-Rips complex}. One of the initial results that successfully used simplicial complexes in the coverage problem of sensor networks was by Silva et al. \citep{silva}, which showed that whether or not the a region $D\subseteq \R^2$ is covered by a sensor network is sufficient to perform a test on the algebraic image of its boundary $\partial D$ in the simplicial chain complex. Using the Čech complex, a subcomplex of Vietoris-Rips complex, Derenick et al. \citep{derenick} gave a distributed motion-planning algorithm for mobile robot team that can collapse the homology. They used a combinatorial method to detect the precise location of the holes and used it to compute the motion plans. 

In this section, we present a more continuous, parameter-free coverage repair algorithm using the weighted combinatorial Laplacian. Since the weighted Laplacian can measure the "almost holes", there does not have to be combinatorially well-defined holes in the simplicial complex. Hence we can take our simplicial complex to be just a contractible complex (a hole-free complex), and thus, as opposed to the case of using Vietoris-Rips complex (or Čech complex), the algorithm does not depend on the particular choice of $\epsilon$. Another advantage of using weighted Laplacian is the fact that the extremal eigenvalues are naturally piece-wise smooth functions of weights, and this gives an obvious candidate for an algorithm, that is, the gradient descent of the extremal eigenvalues with respect to the weights. If we define the weights to be again smooth functions of positions of the mobile robots in the plane, then we have a piece-wise smooth function that measures the significance of holes with respect to the positions of sensors.

Let us now be more specific. Given a set $S$ of points embedded in the Euclidean plane $\R^2$, we construct a 2-dimensional contractible complex $D$ by the \text{Delaunay triangulation}, that is, the triangulation on $S$ such that the open disk bounded by the circumcircle of any triangle of $D$ does not contain any point in $S$. In fact, the use of Delaunay triangulation (or its dual, the \textit{Voronoi partition}) has been a natural choice for sensor coverage problems; for example, see \citep{Cortes}, \citep{Luciano}. We let the weight of each edge be the inverse of the Euclidean length of it in $\R^2$ and then let the weight of a 2-simplex be the product of the weights of its boundary, that is,
\begin{align*}
    w_2(\{v_0, v_1, v_2\}) := w_1(\{v_0, v_1\}) w_1(\{v_0, v_2\}) w_1(\{v_1, v_2\})
\end{align*}
for each 2-simplex $\{v_0, v_1, v_2\} \in D$. Under the assumption that the weight of the edge is always greater than 1 (which is possible by scaling the embedding as necessary), this choice of weight satisfy the filtration condition introduced in Section~\ref{sec:weighted-laplacian}, and thus the resulting weighted simplicial complex $X$ has all the properties shown in the section. We then compute the first weighted Laplacian operator $\Tilde{\mathcal{L}}_1(X)$, its eigenvector $v$ corresponding to the non-zero smallest eigenvalue $\lambda$, and the gradient $d\lambda/d S$ of $\lambda$ with respect to the embedding of $S$. Note that we need $v$ to compute the gradient because $d\lambda / d\Tilde{\mathcal{L}}_1(X)=vv^T$. We can then update the embedding $S$ by the positive scalar multiple of $d\lambda/d S$, and therefore obtains the following algorithm:

\begin{figure}
\centering
\includegraphics[width=0.99\textwidth]{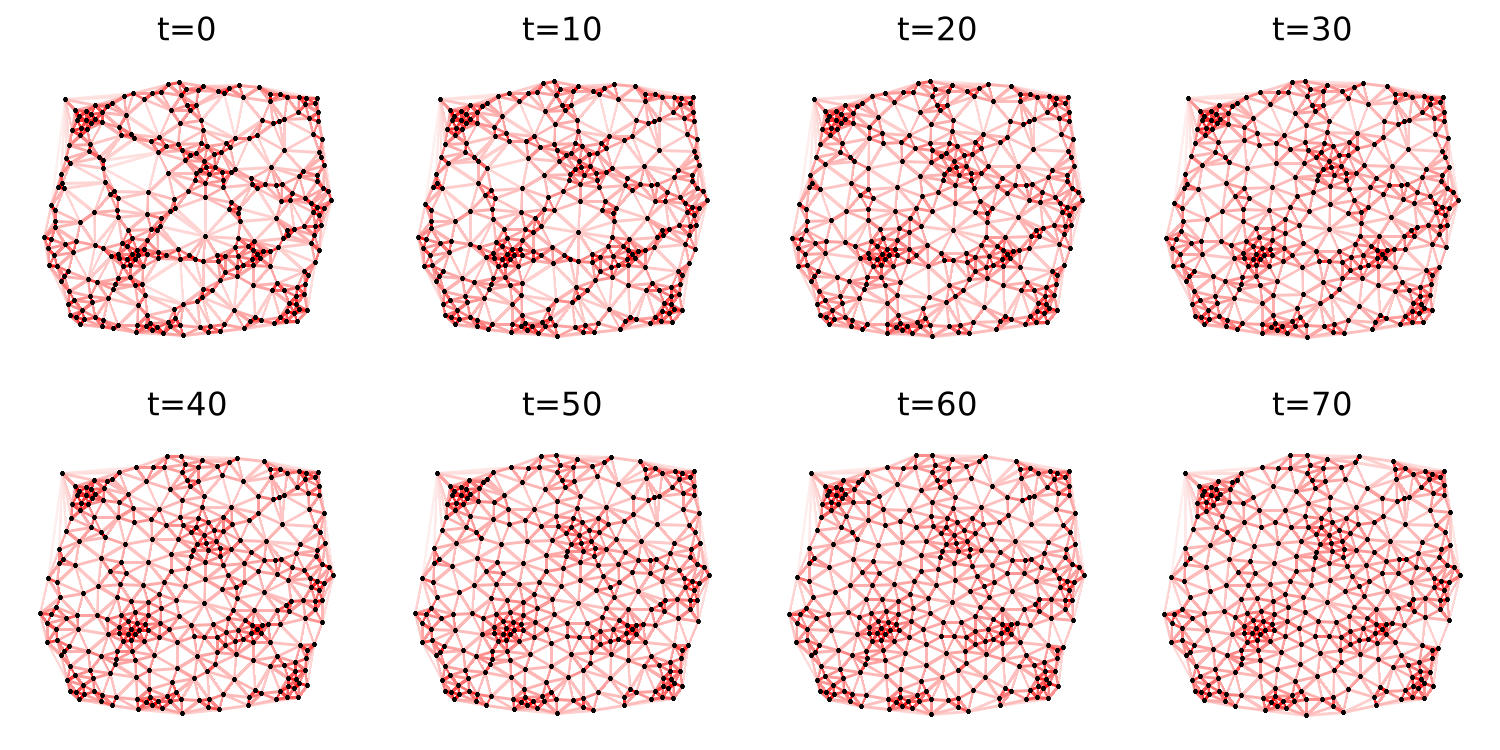}
\caption{A computational result of the proposed coverage repair algorithm with 300 mobile robots. Black points represents the positions of robots and the red line segments show the edges created by the Delaunay triangulation at each routine. Each triangle is filled by a 2-simplex, and thus the complex is a planar, contractible complex. \changedSBa{Note how the complex starts of with multiple almost-holes at $t=0$, but over time the mobile sensors redistribute themselves to reduce the almost-holes in the complex ($t=70$). However, note that there are still some clustering of the robots since in this dynamic coverage repair we are only increasing the smallest eigenvalue of $\Tilde{\mathcal{L}}_1$ and not considering the spectrum of $\Tilde{\mathcal{L}}_0$.} %\todo{Should we add the spectrum below each sub-figure to show how the low eigenvales increase?}
}
\label{coverage repair}
\end{figure}

\begin{algorithm}
\caption{Coverage repair routine with $\Tilde{\mathcal L}_1$}\label{alg:coverage repair in convex domain}
\begin{algorithmic}
\Require Positions of sensors $S \subset \R^2$
\State $D \gets $ \textit{DelaunayTriangulation}$(S)$ 
\State $L \gets \Tilde{\mathcal{L}}_1(D, S)$
\State $v \gets $ The eigenvector that corresponds to the smallest non-zero eigenvalue of $L$
\State Compute the gradient $d\lambda/d S$ using $S$ and $v$, and update $S$ according to the gradient. 
\end{algorithmic}
\end{algorithm} 

Here, \textit{DelaunayTriangulation}$(S)$ is the subroutine that computes the Delaunay triangulation based on the embedding $S$ of points in $\R^2$, and $\Tilde{\mathcal{L}}_1(D, S)$ is the subroutine that computes the weighted Laplacian matrix. A computational result of the algorithm with 80 iterations is shown in Figure \ref{coverage repair}.

\begin{figure}
\centering
\includegraphics[width=0.99\textwidth]{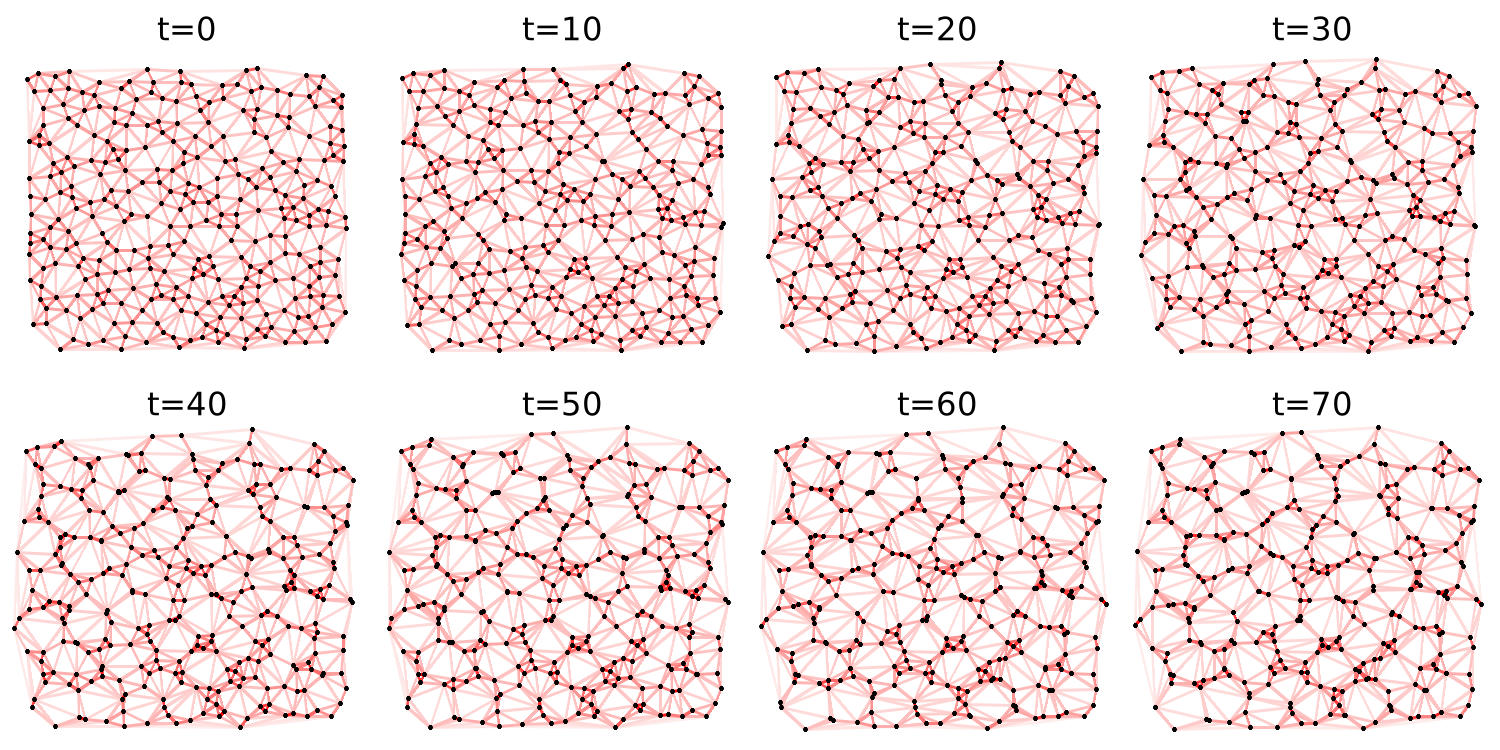}
\caption{A computational result of the caging algorithm with 300 mobile robots with $k=50$. \changedSBa{Starting from a configuration with no well-defined almost-holes,} the algorithm thus tries to create 50 holes in the given complex.}
\label{caging}
\end{figure}

\subsection{\changedSB{Creation of Holes in Sensor Networks for} Caging}
As an inverse approach to our method, we show that it is also possible to create holes in the sensor network. Let $k$ be the number of holes that we would like to create in the simplicial complex. Then it is possible to create $k$-holes in the complex by performing the gradient descent on the $k$th smallest non-zero eigenvalue. A result of this algorithm is shown in Figure \ref{caging}.

\subsection{\changedSBa{Relative Chain Complex and its Application to} Coverage Repair of Sensor Network in a Domain with Obstacles}
It is often the case that the mobile robots need to explore a domain with obstacles. In this situation, however, the Algorithm \ref{alg:coverage repair in convex domain} cannot be used because the Laplacian would not know whether the hole it sees is a hole in the coverage or an obstacle. To handle this issue, we let the Laplacian operator forget about the hole made by the obstacles so that it can make sure the hole it sees is the hole of the coverage. This leads us to the notion of a relative chain complex, and we further define the relative Laplacian operator.

Let $\{C_n\}$ be the chain complex, and let $\{C_n^\prime\}$ be a subcomplex of $\{C_n\}$, that is, each $C_n^\prime$ is a vector subspace of $C_n$, and $c \in C_n^\prime$ implies $\partial_n(c) \in C_{n-1}^\prime$. Then we define the relative chain complex $\{C_n/C_n^\prime\}$ to be the chain complex with the $n$-chains $C_n/C_n^\prime$ and the induced boundary maps $\partial_n: C_n/C_n^\prime \to C_{n-1}/C_{n-1}^\prime$. Because we have not assumed anything beyond $\{C_n\}$ being a chain complex, this construction of the relative complex works for the weighted chain complex we have defined in the previous section. 

Having these notions in mind, we come back to the coverage repair algorithm with obstacles. Let $S = \{s_1,\cdots, s_k\} \subseteq \R^2$ be the set of positions of the $k$ mobile robots distributed in the planar domain. Assume further that there are $m$ obstacles $P_1, \cdots, P_m$ in the domain and that, for each $i=1,\cdots, m$, there is an obstacles $P_i$, represented as a disjoint embedding $\coprod_{i=1}^m D_i$ of discs $D_i$ into $\R^2$. Then we can choose a sufficiently fine discretization $\overline{P}_i$ of the boundary of $P_i$ as a 1-dimensional simplicial complex. Now let $X$ be the simplicial complex build by the Delaunay triangulation over the union of $S$ and the vertices of $\overline{P}_i$'s, and let $Y_i$ be the minimal subcomplex of $X$ containing all the 2-simplexes $X$ that intersects the interior of the obstacle $P_i$. If the $\overline P_i$ is a sufficiently fine discretization of the boundary of $P_i$, then we can guarantee that $P_i \subset Y_i$. We can now recursively compute the relative chain complex $C^i=\{ C^{i-1} / C_n Y_i \}$, with the base case $C^0 = C_n X$. Then the Laplacian $\Tilde{ \mathcal L_1}(C^m)$ over the resulting chain complex $C^m$ is the weighted Laplacian operator that forgets about the holes created by obstacles. We know that $\Ker \Tilde{ \mathcal L_1}(C^m) = 0$, because for each $i=1,\cdots,m$, the following part of the long exact sequence
\begin{center}
\begin{tikzcd}
    0 = \Tilde H_1(C^{i-1}) \arrow{r} & \Tilde H_1(C^i) \arrow{r}{\Tilde{\mathcal B}_{1}}& \Tilde H_0(C_n Y_i) = 0
\end{tikzcd}
\end{center}
yields that $\Tilde H_1(C^i)=0$. Thus the smallest eigenvalue of $\Tilde{ \mathcal L_1}(C^m)$ is non-zero, and therefore we have obtained the following algorithm:
\begin{algorithm}
\caption{Coverage Repair Routine in a Domain with Obstacles Using $\Tilde{\mathcal L}_1$}\label{alg:coverage repair with obstacle}
\begin{algorithmic}
\Require Positions of mobile robots $S \subset \R^2$, Discritization $\overline{P}_i$ of the boundary of the obstacles $P_i$ ($i=1,\cdots, m$)
\State $S \gets S \cup \bigcup_{i=1}^m $\textit{VerticesOf}$(\overline{P}_i)$
\State $D \gets $ \textit{DelaunayTriangulation}$(S)$ 
\ForAll {$i \in \{1,\cdots,m\}$}
    \While {$\overline{P}_i \not \subset Y_i$ }
        \State $\overline{P}_i \gets $ \textit{Subdivide}$(\overline{P}_i)$
        \State $S \gets S \cup \bigcup_{i=1}^m $\textit{VerticesOf}$(\overline{P}_i)$
        \State $D \gets $ \textit{DelaunayTriangulation}$(S)$ 
    \EndWhile
\EndFor
\State $L \gets \Tilde{\mathcal{L}}_1(D, S)$
\State $v \gets $ The eigenvector that corresponds to the smallest eigenvalue of $L$
\State $d\lambda d S =$ Gradient of $\lambda$ with respect to $S$
\State Update $S$ according to $d\lambda d S$ and the ruspulsive gradient from the obstacles  $\sum R(P_i)$.
\end{algorithmic}
\end{algorithm}

Here, \textit{VerticesOf($A$)} is the function that takes in a simplicial complex $A$ and spits out the set of vertices of $A$, and $R(P_i)$ is the repulsive gradient that tries to keep mobile robots away from the obstacles $P_i$ and is nonzero only near the boundary of $P_i$. Comparing this algorithm with the Algorithm~\ref{alg:coverage repair in convex domain}, the repulsive gradient $R(P_i)$ might seem a little unnatural. However, it is necessary; otherwise, some mobile robots will try to get inside the obstacles because the relative Laplacian operator forgets the concrete geometry of the obstacles.
\changedSBa{A numerical result from this algorithm is presented in Figure~\ref{fig:result-obstacle}.}

\begin{figure}
\centering
\includegraphics[width=0.99\textwidth]{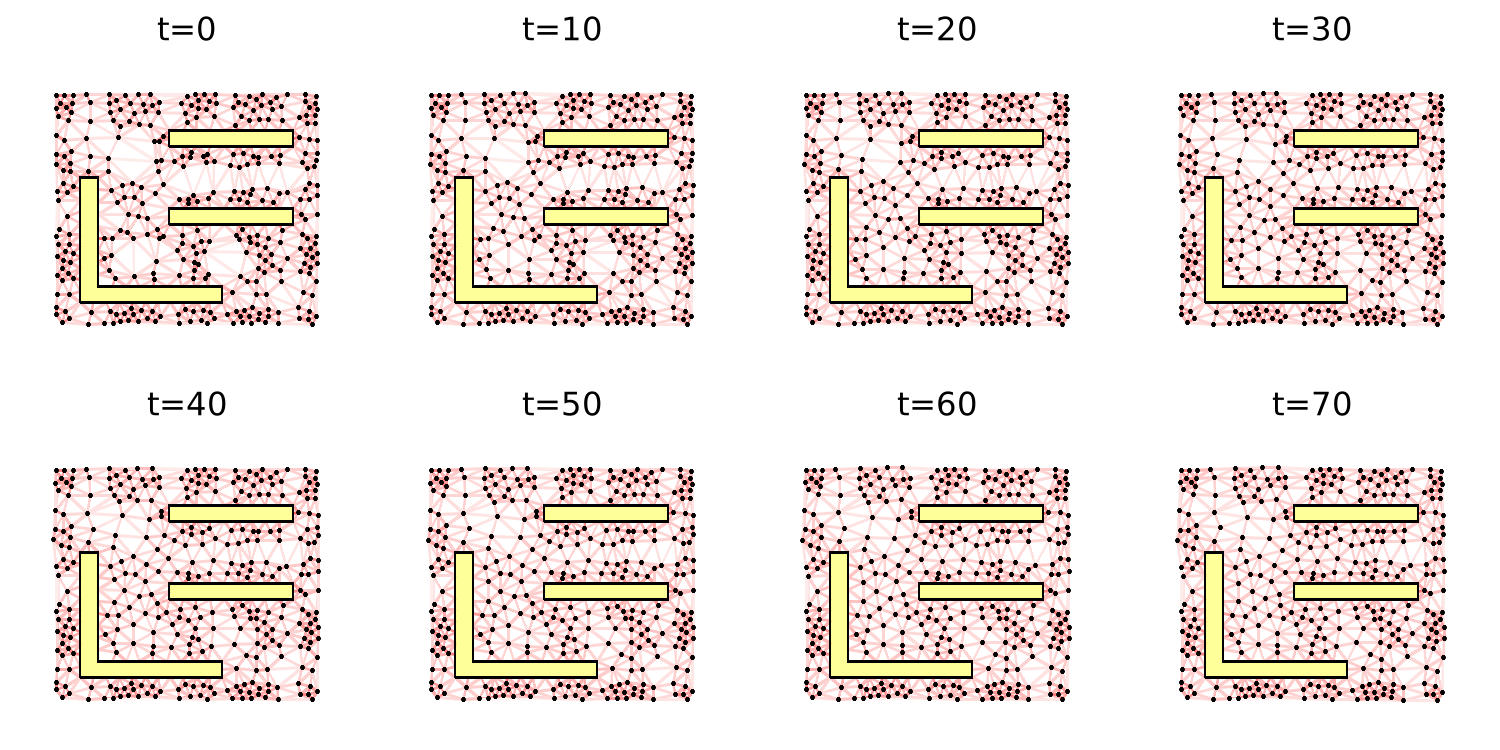}
\caption{A computational result of the coverage repair algorithm \changedSBa{in an environment} with obstacles. The 250 mobile robots are represented as the black points. There are three obstacles represented as the yellow polygons.}
\label{fig:result-obstacle}
\end{figure}

%\printbibliography

\appendix

\section{Proofs} \label{appendix:proofs}
\begin{proof}[Proof of Proposition \ref{fundamentals of laplacian}]
	
\begin{enumerate}[label=(\alph*),leftmargin=2em]
    \item We show $\Ker({\mathcal B}_n^\ast {\mathcal B}_n) \subseteq \Ker {\mathcal B}_n$ only, because the other inclusion is trivial. Suppose that ${\mathcal B}^\star {\mathcal B}_n(x) = 0$ for some $x\in C_n$. Then we must have $0=\langle x, {\mathcal B}_n^\ast {\mathcal B}_n x \rangle = \langle {\mathcal B}_n x, {\mathcal B}_n x \rangle$ by the self-adjointness, hence ${\mathcal B}_n x = 0$. The proof of the other statement is similar.

    \item We show $\Ker ({\mathcal B}_n^\ast {\mathcal B}_n + {\mathcal B}_{n+1} {\mathcal B}_{n+1}^\ast) \subseteq \Ker({\mathcal B}_n^\ast {\mathcal B}_n) \cap \Ker({\mathcal B}_{n+1} {\mathcal B}_{n+1}^\ast)$ only, because the other inclusion is trivial. If $x\in \Ker({\mathcal B}_n^\ast {\mathcal B}_n + {\mathcal B}_{n+1} {\mathcal B}_{n+1}^\ast)$, then by the bilinearity of the inner product, we have $$0 = \langle ({\mathcal B}_n {\mathcal B}_n^\ast + {\mathcal B}_{n+1} {\mathcal B}_{n+1}^\ast)x, x \rangle = \langle {\mathcal B}_n^\ast {\mathcal B}_n x, x \rangle + \langle {\mathcal B}_{n+1} {\mathcal B}_{n+1}^\ast x, x \rangle = \langle {\mathcal B}_n x, {\mathcal B}_n x \rangle + \langle {\mathcal B}_{n+1}^\ast x, {\mathcal B}_{n+1}^\ast x \rangle.$$ The sum of two non-negative numbers being zero implies all summands are zero, and it is the case here because the inner product of two identical vectors are non-negative. Hence we conclude that ${\mathcal B}_n x=0$ and ${\mathcal B}_{n+1}^\ast x=0$.

    \item This is the direct consequence of ${\mathcal B}_n {\mathcal B}_{n+1}=0$ and ${\mathcal B}_{n+1}^\ast {\mathcal B}_n\ast=0$
\end{enumerate}
\end{proof}

\begin{proof}[Proof of Theorem \ref{Discrete Hodge Theorem}]
    For each $n$, we have
    \begin{align*}
        \Ker \mathcal L_n &= \Ker (\mathcal L_n^{\textit{down}} + \mathcal L_n^{\textit{up}}) \\
        &= \Ker \mathcal L_n^{\textit{down}} \cap \Ker \mathcal L_n^{\textit{up}} \\
        &= \Ker {\mathcal B}_n \cap \Ker {\mathcal B}^\ast_{n+1} \\
        &= \Ker {\mathcal B}_n \cap (\Img {\mathcal B}_{n+1})^\perp \\
        &\cong H_n\,C.
    \end{align*}
\end{proof} 

\begin{proof} [Proof of Proposition \ref{Norm}]
    Recall that all the results about the combinatorial Laplacian stated in Section~\ref{sec:background} applies to our weighted Laplacian as well, and from the Proposition \ref{fundamentals of laplacian}(c) we can deduce that $\| \Tilde{\mathcal{L}}_n \|$ is equal to either $\| \Tilde{\mathcal{L}}_n^{\textit{down}} \|$ or $\| \Tilde{\mathcal{L}}_n^{\textit{up}} \|$. Hence,
    \begin{equation*}
        \| \Tilde{\mathcal L}_n \| \leq \max \left\{ \| \Tilde{\mathcal{L}}_n^{\textit{down}} \|,\; \| \Tilde{\mathcal{L}}_n^{\textit{up}} \| \right\} = \max\left\{ \| \Tilde {\mathcal B}_n \|^2,\; \| \Tilde {\mathcal B}_{n+1} \|^2 \right\} \leq \max\left\{ \| \Tilde {\mathcal B}_n \|_F^2,\; \| \Tilde {\mathcal B}_{n+1} \|_F^2 \right\}.
    \end{equation*}
    Hence it suffices to show $\| \Tilde {\mathcal B}_n \|_F^2 \leq |X_n|(n+1)$ for each $n\geq 0$. In the case of $n=0$, this is trivial. If $n>0$, the weighted boundary matrix has the concrete form
    \begin{align*}
        \Tilde {\mathcal B}_{n} (c) = \sum_{b\in X_{n-1}: \, b \subseteq c} \frac{w_n(c)}{w_{n-1}(b)} b. 
    \end{align*}
    for $c \in X_n$. Therefore, by the filtration condition, we obtain
    \begin{align*}
        \| \Tilde{ \mathcal B} \|_F^2 \leq \sum_{c\in X_n} \; \sum_{b\in X_{n-1}:\, b\subseteq c} \left( \frac{w_{n}(c)}{w_{n-1}(b)} \right)^2 \leq \sum_{c\in X_n} \sum_{b\in X_{n-1}} 1 = |X_n| (n+1),
    \end{align*}
    which completes the proof.
\end{proof}

\begin{proof}[Proof of \ref{norm of a union}]
    Note that we have the following equation;
    \begin{align*}
        \Tilde{\mathcal L}_{\Tilde X \cup \Tilde Y}^{\textit{up}} v = \iota_{\Tilde X} \Tilde{\mathcal L}_{\Tilde X}^{\textit{up}} \pi_{\Tilde X}(v) + \iota_{\Tilde Y} \Tilde{\mathcal L}_{\Tilde Y}^{\textit{up}} \pi_{\Tilde Y}(v) - \iota_{\Tilde X \cap \Tilde Y} \Tilde{\mathcal L}^{\textit{up}}_{\Tilde X \cap \Tilde Y} \pi_{\Tilde X \cup \Tilde Y}(v),
    \end{align*}
    where $\iota_{A}: A \hookrightarrow \Tilde X\cup \Tilde Y$ are natural inclusion maps and $\pi_{A}: \Tilde X\cup \Tilde Y \to A$ are natural projection maps, for $A=\Tilde X,\Tilde Y,\Tilde X\cap \Tilde Y$. We have that $\Tilde{\mathcal L}^{\textit{up}}_{\Tilde X\cap \Tilde Y} = 0$, since $\Tilde X \cap \Tilde Y$ does not have any $(n+1)$-simplex by assumption. Furthermore, the inclusion maps preserve the norm, hence
    \begin{align*}
        \| \Tilde{\mathcal L}^{\textit{up}}_{\Tilde X \cup \Tilde Y} v\| = \| \iota_{\Tilde X} \Tilde{\mathcal L}^{\textit{up}}_{\Tilde X} \pi_{\Tilde X}(v) + \iota_{\Tilde Y} \Tilde{\mathcal L}^{\textit{up}}_{\Tilde Y} \pi_{\Tilde Y} (v) \| \leq \| \Tilde{\mathcal L}^{\textit{up}}_{\Tilde X} \pi_{\Tilde X}(v) \| + \| \Tilde{\mathcal L}^{\textit{up}}_{\Tilde Y} \pi_{\Tilde Y} (v) \|.
    \end{align*}
\end{proof}

\begin{proof}[Proof of \changedSB{Lemma}~\ref{inequality about a union}]
    Choose a total ordering of $V$, then it will induce a total ordering on $X_n$, for which we will denote $\prec$. For $c_1 \prec \cdots \prec c_{|X_n|} \in X_n$, write $v = \alpha_1 c_1 + \cdots + \alpha_{|X_n|} c_{|X_n|}$ for some $\alpha_i \in \R$ ($1 \leq i \leq |X_n|$). By assumption, $\alpha_i=0$ for $c_i \not \in \partial_n X$. Consider the set $F = \{f\in X_{n+1}: f \supseteq c \text{ for some } c \in (\partial_n X) \cap X_n \}$. Now define a permutation $\tau$ on $X_n$ by 
    \begin{align*}
        \tau(c_i) = 
        \begin{cases}
            \min\{c_j \subseteq f: c_j \succ c_i\} & \text{if $c_i \subseteq f$ for some $f\in F$ and $i < n$} \\
            \min\{c_j \subseteq f\} & \text{if $c_i \subseteq f$ for some $f\in F$ and $i = n$} \\
            c_i & \text{otherwise}.
        \end{cases}
    \end{align*}
    $\tau$ is well-defined because if $c_i \subseteq f_1$ and $c_i \subseteq f_2$ for some $f_1, f_2\in F$ then $f_1=f_2$ by assumption, and it is clear that $\tau$ is a cyclic permutation of order $n+1$, that is, $\tau^{n+1}(c_i) = c_i$ for each $i$. On the other hand, we have
    \begin{align*}
        \left \| \mathcal L^{\textit{up}} v \right \| = \left \| \sum_i \sigma(c_i,f_i) \sum_{j: c_i \cup c_j \subseteq f_i} \sigma(c_j,f_i) \alpha_j c_i \right \| \leq \left \| \; \sum_{i: \alpha_i \neq 0} \sum_{j: c_j \subseteq f_i} |\alpha_j| c_i \right \|.
    \end{align*}
    Now notice that
    \begin{align*}
        \sum_i \sum_{j: c_j \subseteq f_i} |\alpha_j| c_i = \hat{v} + \tau \hat{v} + \cdots + \tau^n \hat{v},
    \end{align*}
    where $\hat v = |\alpha_1| c_1 + \cdots + \left| \alpha_{|X_n|} \right| c_{|X_n|}$. $\tau$ is only a permutation of the basis elements, thus it preserves the norm. Furthermore, $\|v\| = \|\hat v\|$. Hence,
    \begin{align*}
        \| \mathcal L^{\textit{up}} v \| \leq \| \hat{v} \| + \| \tau \hat{v} \| + \cdots + \| \tau^n \hat{v} \| = (n+1)\| \hat v \| = (n+1)\| v \|.
    \end{align*}
\end{proof}

\begin{proof}[Proof of \changedSB{Theorem}~\ref{main theorem}]
    Let $\Tilde X$, $\Tilde Y$ be as above. We have the Mayer-Vietoris sequence about the union $\Tilde X \cup \Tilde Y$, which is an exact sequence of the form
    \begin{center}
    \begin{tikzcd}
        \cdots \arrow{r}& \Tilde H_{n+1}(\Tilde X \cup \Tilde Y) \arrow{r}{{\mathcal B}_{n+1}}& \Tilde H_n(\Tilde X \cap \Tilde Y) \arrow{r}& \Tilde H_n(\Tilde X) \oplus \Tilde H_n(\Tilde Y) \arrow{r}& \Tilde H_n(\Tilde X \cup \Tilde Y) \arrow{r}{{\mathcal B}_n}& \cdots 
    \end{tikzcd}
    \end{center}
    (See \citep{hatcher}). By assumption (2) and (3), we find a cycle $w$ in $\Tilde C_n \Tilde X \cap \Tilde C_n \Tilde Y$ that is trivial in $\Tilde C_n \Tilde Y \subseteq \Tilde C_n X \cup \Tilde C_n Y$ but is nontrivial in both $\Tilde C_n \Tilde X$ and $\Tilde C_n \Tilde X \cap \Tilde C_n \Tilde Y$. Then let $v$ be the unit vector in $(\Tilde {\mathcal B}_{n+1} (\Tilde C_{n+1} \Tilde X))^\perp$ obtained by applying the Gram-Schmidt process to $w$ with respect to the basis vectors of $\Tilde {\mathcal B}_{n+1} (\Tilde C_{n+1} \Tilde X)$. We show that this vector $v$ is the desired vector.

    Since $v$ is sum of the cycle $w$ and boundaries in $\Tilde C_n \Tilde X$, $v$ itself is a cycle in $\Tilde C_n X$, which implies that $\Tilde {\mathcal B}_{\Tilde X \cup \Tilde Y} (v) = 0$.
    Furthermore, since $v$ is orthogonal to the image of $\Tilde C_{n+1} \Tilde X$ under ${\mathcal B}$, we have $\Tilde {\mathcal B}^* v = 0$. Hence, using lemma \ref{norm of a union}, we have
    \begin{equation*}
        \| \Tilde {\mathcal L}_{\Tilde X \cup \Tilde Y} (v) \| = \| \Tilde{\mathcal L}_{\Tilde X \cup \Tilde Y}^{\textit{up}}(v) \| \leq \| \Tilde{\mathcal L}_{\Tilde X}^{\textit{up}}(\pi_{\Tilde X}(v)) + \Tilde{\mathcal L}_{\Tilde Y}^{\textit{up}}(\pi_{\Tilde Y}(v)) \| = \| \Tilde{\mathcal L}_{\Tilde Y}^{\textit{up}}(\pi_{\Tilde Y}(v)) \|.
    \end{equation*}
    Now it suffices to show $\Tilde{\mathcal L}_{\Tilde Y}^{\textit{up}}(v) \leq \epsilon(n+1)$. By assumption (4), each basis element $c \in \Tilde C_n \Tilde Y$ with $\langle c, v \rangle \neq 0$ has at most one cell $f\in \Tilde C_{n+1} \Tilde Y$ such that $\langle c, {\mathcal B}(f) \rangle \neq 0$. Thus, if $\pi_{\Tilde Y}(v) = \alpha_1 c_1 + \cdots + \alpha_k c_k$ where each $c_i$ is a basis element of $\Tilde C_n \Tilde Y$, then for each $i$ with $\alpha_i \neq 0$, there is a unique $f_i$ such that $\langle c_i, {\mathcal B} f_i \rangle \neq 0$. Therefore,
    \begin{align*}
        \left \| \Tilde{\mathcal L}_{\Tilde Y}^{\textit{up}}(\pi_{\Tilde Y}(v)) \right \| &= \left \| \sum_{i=1}^k \sum_{f: f \supseteq c_i} \; \sum_{c_j: c_j \subseteq f } \sigma(c_i, c_j) \frac{w(f)^2}{w(c_i)w(c_j)} \alpha_j c_i \right\| \\
        &= \left \| \sum_{i: \alpha_i \neq 0} \sum_{c_j: c_j \subset f_i } \sigma(c_i, c_j) \frac{w(f_i)^2}{w(c_i)w(c_j)} \alpha_j c_i \right\| \\
        &\leq \left \| \sum_{i=1}^k \sum_{c_j: c_j \subset f_i } \sigma(c_i, c_j) \frac{w(f_i)}{w(c_i)} \alpha_j c_i \right\| & \text{(by the assumption (1))} \\
        &\leq \epsilon \left \| \sum_{i=1}^k \sum_{c_j: c_j \subset f_i } \sigma(c_i, c_j) \alpha_j c_i \right \| & \text{(by the assumption (5))} \\
        &= \epsilon \left \| \mathcal L^{\textit{up}} \pi_{\Tilde Y}(v) \right \| \\
        &\leq \epsilon (n+1) \left \| \pi_{\Tilde Y}(v) \right \| & \text{(by Lemma ~\ref{inequality about a union})}\\
        &\leq \epsilon (n+1),
    \end{align*}
    where the last inequality is by the fact that $v$ is a unit vector. 
\end{proof}

\begin{proof}[Proof of \ref{extended main theorem}]
    For each union operation $(\Tilde X \cup \bigcup_{j=1}^{i-1} \Tilde Y_j) \cup \Tilde Y_i$, choose $w_i$ as we have chosen $w$ in the proof of Theorem $\ref{main theorem}$. Then $w_i/\|w_i\|$ forms an orthonormal basis of $k$-dimensional vector space of $\Tilde C_n \Tilde X$, because $Y_i$'s are pairwise disjoint and $w_i$ takes nonzero values only on $\Tilde Y_i$. By assumption, we can deduce that $\dim (H_n X) \geq k$, and by definition we have $\dim \left(\Ker \Tilde {\mathcal B}_n \cap (\Img \Tilde {\mathcal B}_{n+1})^\perp \right) \geq k$. Hence we can choose an orthogonal transformation $Q: \Ker {\mathcal B}_n \to \Ker {\mathcal B}_n$ such that $v_i := Q(w_i/\|w_i\|) \in (\Tilde {\mathcal B}_{n+1} \; \Tilde C_{n+1})^\perp$ for all $1 \leq i \leq k$. By the choice of $Q$, it is automatic that the equation (\ref{orthogonal condition for v_i}) holds, that $v_i$ is a unit vector, and that $v_i$ is a cycle in $\Tilde C_n(\Tilde X \cup \bigcup_{j=1}^{k} \Tilde Y_j)$. Now we can apply the same argument about $v$ in Theorem \ref{main theorem} to each $v_i$ to verify that (\ref{successive upper bound}) is satisfied.
\end{proof}

\vspace{3em}
\bibliography{citation}

\end{document}